\documentclass[11pt,reqno]{amsart}

\usepackage[table]{xcolor} 
\usepackage[utf8]{inputenc}
\usepackage[T1]{fontenc}
\usepackage{amsmath,amsfonts}
\usepackage{algorithmic}
\usepackage{algorithm}
\usepackage{array}
\usepackage{textcomp}
\usepackage{stfloats}
\usepackage{url}
\usepackage{stmaryrd}
\usepackage{graphicx}
\usepackage{cite}
\usepackage{mathrsfs}
\usepackage{amsthm}
\usepackage{amssymb}
\usepackage{microtype}
\usepackage{subcaption}  
\usepackage{textcomp}
\usepackage{hyperref}
\hypersetup{urlcolor=blue, citecolor=red}
\usepackage[rightcaption]{sidecap}
\usepackage{tikz}
\usetikzlibrary{arrows.meta,positioning,fit,calc}

\usepackage{fullpage}
\usepackage{pgfplots}
\usepackage[siunitx]{circuitikz}
\pgfplotsset{compat=1.15}
\usepackage{mathrsfs}

\DisableLigatures{encoding = *, family = * }

\usepackage{cleveref}
\usepackage{mathtools}
\usepackage{amsopn}

\begingroup 
\theoremstyle{plain}
\newtheorem{theorem}{Theorem}[section]
\newtheorem{proposition}[theorem]{Proposition}
\newtheorem{lemma}[theorem]{Lemma}
\theoremstyle{definition}
\newtheorem{definition}[theorem]{Definition}
\newtheorem{remark}[theorem]{Remark}
\endgroup

\numberwithin{table}{section}
\numberwithin{equation}{section}
\numberwithin{figure}{section}

\def \*#1{\boldsymbol{#1}}
\newcommand{\NN}{\mathbb{N}}

\newcommand{\RR}{\mathbb{R}}

\newcommand{\MC}[1]{\mathcal{#1}}
\newcommand{\inter}[1]{\llbracket #1 \rrbracket}

\definecolor{capri}{rgb}{0.0, 0.75, 1.0}
\definecolor{carnationpink}{rgb}{1.0, 0.65, 0.79}
\definecolor{brightgreen}{rgb}{0.4, 1.0, 0.0}
\definecolor{softgray}{RGB}{245,246,248}

\title{Interior interpretability with attention rollout: contraction and propagation profiles in Transformers}

\author{Umberto Biccari\textsuperscript{\,$\dagger$}}
\address{\textsuperscript{$\dagger$}\, Chair of Computational Mathematics, DeustoTech, University of Deusto, Avenida de las Universidades 24, 48007 Bilbao, Basque Country, Spain.} 
\email{umberto.biccari@deusto.es}

\thanks{This project has received funding from the European Research Council (ERC) under the European Union's Horizon Europe research and innovation programme (grant agreement No. 101096251, CoDeFeL). This material is based upon work supported by the Air Force Office of Scientific Research under award number FA8655-22-1-7012. EZ was partially supported by the Alexander von Humboldt Professorship program; the European Union's Horizon Europe MSCA project ModConFlex (HORIZON-MSCA-2021-DN-01, project 101073558); the Transregio 154 Project Mathematical Modelling, Simulation and Optimization Using the Example of Gas Networks of the DFG; and SURE-AI: The Norwegian Centre for Sustainable, Risk-Averse, and Ethical AI, grant 357482, Research Council of Norway. UB and EZ were partially supported by the Grant PID2023-146872OB-I00-DyCMaMod of MICIU (Spain) and by the COST Actions CA24122-Multiscale Stochastics, Patterns, and Analysis of Combinatorial Environments and CA24136-Interactions between Control Theory and Machine Learning. QH was supported by the RYC2024-048848-I grant funded by MICIU/AEI/10.13039/501100011033 and the FSE+.}

\author{Qian Huang\textsuperscript{\,$\ast$ \, $\ddagger$}} 
\address{\textsuperscript{$\ast$}\, Universidad Carlos III de Madrid, Gregorio Mill\'an Institute for Fluid Dynamics, Nanoscience and Industrial Mathematics.
\newline \indent \textsuperscript{$\ddagger$} Universidad Carlos III de Madrid, and Universidad Carlos III de Madrid, Department of Mathematics, Legan\'es, 28911, Spain.} 
\email{qhuang@inst.uc3m.es}

\author{Enrique Zuazua\textsuperscript{\,$\dagger$\,$\P$\,$\S$}}  
\address{\textsuperscript{$\dagger$}\, Chair of Computational Mathematics, DeustoTech, University of Deusto, Avenida de las Universidades 24, 48007 Bilbao, Basque Country, Spain.} 
\address{\textsuperscript{\,$\P$} Chair for Dynamics, Control, Machine Learning, and Numerics (Alexander von Humboldt-Professorship), Department of Mathematics, Friedrich-Alexander-Universit\"at Erlangen-N\"urnberg, 91058 Erlangen, Germany}
\address{\textsuperscript{\,$\S$} Universidad Aut\'onoma de Madrid, Departamento de Matem\'aticas, Ciudad Universitaria de Cantoblanco, 28049 Madrid, Spain.}
\email{enrique.zuazua@deusto.es, enrique.zuazua@fau.de, enrique.zuazua@uam.es}

\begin{document}

\begin{abstract}
Feature-attribution methods assign scores relating input variables to a model's output, but do not by themselves characterize how explicitly defined interaction operators compose across its intermediate layers. We introduce \emph{interior interpretability}, a propagation-based perspective on internal model organization, and instantiate it for tabular Transformers using attention rollout. We interpret rollout as a row-stochastic operator encoding attention-mediated propagation between feature tokens. By applying classical Doeblin--Dobrushin contraction theory, we show that a rollout operator with a small Dobrushin coefficient is quantitatively close to a rank-one stochastic matrix whose common row is determined by its normalized column sums. This result gives a structural interpretation to the corresponding rollout propagation profile. In Transformers trained for metabolomic age prediction, the measured rollout contraction strengthens with depth. Trained and randomly initialized models also exhibit different propagation profiles, although the present experiments do not establish the predictive relevance of individual rollout-ranked variables. Exploratory comparisons with PCA and GradientExplainer approximations to SHAP reveal localized agreement among highly ranked variables but weak agreement across complete rankings. Attention rollout is therefore used here as a diagnostic of attention-mediated propagation, not as a causal explanation or faithful attribution of the complete Transformer.
\end{abstract}

\maketitle

\section{Introduction and motivation}

Methods for feature attribution address a central interpretability question: how do the input variables of a learned predictor contribute to its output? For a model $f:\RR^d\to\RR$, methods such as SHAP, Integrated Gradients, and saliency maps analyze the map $x\mapsto f(x)$ and assign a score to each coordinate of $x$ \cite{arunika2024survey,lundberg2017unified,molnar2020interpretable,mosca2022shap,simonyan2013deep,sundararajan2017axiomatic,zhang2021survey}. These methods provide output-oriented explanations, but they do not by themselves describe how explicitly identifiable interaction operators accumulate across the intermediate layers of a deep architecture.

This paper studies that complementary question. We use the term \emph{interior interpretability} for the analysis of propagation operators associated with intermediate model interactions. The objective is deliberately narrower than reconstructing the complete hidden-state dynamics: a propagation operator captures only the part of the computation represented by its construction. In the Transformer setting considered here, the relevant operators are obtained from self-attention and describe attention-mediated exchange between feature tokens. They do not include value and output projections, normalization, MLP sublayers, or all nonlinear transformations of the complete model.

Attention rollout \cite{abnar2020quantifying} provides a natural object for this analysis. It composes the attention operators of successive layers, including the standard residual averaging, into a row-stochastic matrix. For tabular Transformers, where individual variables are represented by dedicated tokens, the resulting matrix describes how attention-mediated propagation between source and destination features accumulates with depth. This stochastic-operator viewpoint leads to two distinct questions:
\begin{enumerate}
    \item What structural form does the rollout operator acquire as depth increases?
    \item Which properties of its propagation profile follow from the rollout construction, and how do the observed profiles differ between trained and randomly initialized models?
\end{enumerate}

The first question is mathematical. Products of stochastic matrices are governed by classical contraction theory, suggesting that the rows of a sufficiently contractive rollout operator should approach a common profile. The rollout product supplies a layerwise contraction bound, but this bound does not by itself guarantee decay with depth. The second question is empirical. The common profile need not be uniform and depends on the realized attention matrices. Separating the mathematical constraints of rollout composition from profile differences observed after optimization is therefore essential: rank-one organization alone is not evidence of learning or task relevance.

Figure \ref{fig:interpretability_scheme} places this propagation viewpoint alongside two reference analyses used in the empirical study. PCA describes variance structure in the input data, GradientExplainer provides an expected-gradients approximation to prediction-attribution scores, and rollout describes the propagation geometry encoded by the attention operators. These are not three interchangeable notions of explanation; they answer different questions and are compared to determine where their rankings agree or diverge.

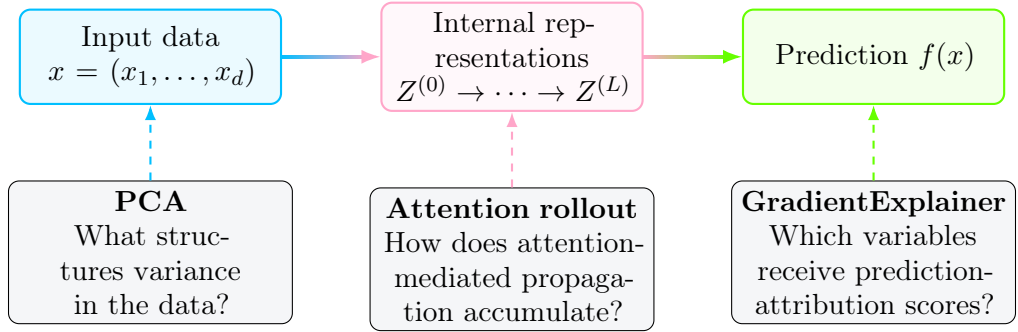
\begin{figure*}
    \centering
    \begin{tikzpicture}[
        font=\rmfamily,
        node distance=1.2cm and 1.3cm,
        main/.style={draw, rounded corners, thick, align=center, minimum height=1.25cm, text width=3.2cm},
        lens/.style={draw, rounded corners, align=center, minimum height=0.9cm, text width=3.5cm, fill=softgray},
        arrow/.style={-{Latex[length=2.5mm]}, thick},
        link/.style={-{Latex[length=2mm]}, dashed, thick}
    ]
    \node[main, draw=capri, fill=capri!8] (data) {Input data\\$x=(x_1,\ldots,x_d)$};
    \node[main, draw=carnationpink, fill=carnationpink!8, right=of data] (hidden) {Internal representations\\$Z^{(0)}\to\cdots\to Z^{(L)}$};
    \node[main, draw=brightgreen, fill=brightgreen!8, right=of hidden] (output) {Prediction $f(x)$};
    
    \shade[left color=capri,right color=carnationpink]
        ($(data.east)+(0,-0.55pt)$)
        rectangle
        ($(hidden.west)+(-0.14cm,0.55pt)$);
    \draw[-{Latex[length=2.5mm]}, thick, carnationpink]
        ($(hidden.west)+(-0.18cm,0)$) -- (hidden.west);
    \shade[left color=carnationpink,right color=brightgreen]
        ($(hidden.east)+(0,-0.55pt)$)
        rectangle
        ($(output.west)+(-0.14cm,0.55pt)$);
    \draw[-{Latex[length=2.5mm]}, thick, brightgreen]
        ($(output.west)+(-0.18cm,0)$) -- (output.west);
    
    \node[lens, below=1.0cm of data] (pca) {\textbf{PCA}\\What structures variance in the data?};
    \node[lens, below=1.0cm of hidden] (rollout) {\textbf{Attention rollout}\\How does attention-mediated propagation accumulate?};
    \node[lens, below=1.0cm of output] (shap) {\textbf{GradientExplainer}\\Which variables receive prediction-attribution scores?};
    
    \draw[link, capri] (pca) -- (data);
    \draw[link, carnationpink] (rollout) -- (hidden);
    \draw[link, brightgreen] (shap) -- (output);
    \end{tikzpicture}
    \caption{Three complementary analytical viewpoints used in this work. PCA describes variance structure in the input data, attention rollout describes attention-mediated propagation across intermediate layers, and GradientExplainer provides an expected-gradients approximation to prediction-attribution scores. These quantities are not interchangeable and need not induce the same feature rankings.}
    \label{fig:interpretability_scheme}
\end{figure*}

The paper makes four contributions:
\begin{enumerate}
    \item We formulate interior interpretability as the analysis of propagation operators associated with intermediate model interactions and instantiate this perspective using attention rollout in tabular Transformers.
    \item Using a direct finite-dimensional consequence of classical Doeblin--Dobrushin theory \cite{dobrushin1956central}, we characterize the rank-one regime of rollout operators and identify their normalized column sums with the corresponding mean row profile.
    \item We place the layerwise contraction bound for rollout composition alongside empirical differences between trained and randomly initialized propagation profiles.
    \item We compare rollout propagation rankings with PCA and GradientExplainer approximations to SHAP in a metabolomic age-prediction case study. These comparisons are descriptive and exploratory: rollout is treated as a diagnostic of attention-mediated propagation rather than as a causal or faithful explanation of the complete predictor.
\end{enumerate}

The empirical setting is metabolomic age prediction, an application of current interest in precision medicine \cite{hannum2013genome,horvath2013dna,ibanez2025metabolomic}. The real-world study is complemented by a synthetic example used only as a second qualitative illustration. Across the trained architectures, rollout contraction strengthens with depth. The trained--random-initialization comparisons address a different question: the two groups exhibit different propagation profiles and rankings in the reported runs. Agreement with GradientExplainer scores is localized among some highly ranked variables, while complete rankings remain weakly correlated.

The present paper is organized as follows. Section \ref{sec:attention} introduces the analytical framework underlying the proposed interpretability approach. In particular, we present the attention rollout operator and its structural characterization. In Section \ref{sec:methods}, we describe the dataset and models used in our empirical study, whose results are presented in Section \ref{sec:experiments}. We conclude with a discussion of the implications of our findings and potential directions for future research in Section \ref{sec:conclusions}. Appendix \ref{app:proof} contains the proof of the structural result and additional mathematical details supporting our analysis. Finally, Appendix \ref{app:bio} places selected highly ranked variables in biomedical context.

\subsection{Positioning with respect to the existing literature}\label{subsec:literature}

The perspective developed here belongs to the broader effort to understand internal model mechanisms beyond input--output behavior. Related work includes mechanistic interpretability and neural-circuit analysis \cite{elhage2022toy,nanda2023progress,olah2020zoom,rauker2023toward}, as well as methods that probe learned representations \cite{alain2016understanding,hewitt2019designing}. These approaches study different objects and levels of explanation. Our narrower focus is the structure of an explicitly defined propagation operator derived from attention matrices.

Within this broader perspective, attention rollout is interpreted as a global stochastic operator describing how attention-mediated interactions compose across successive Transformer layers. This operator does not represent the complete evolution of the hidden states.

Our work is also related to the extensive literature discussing whether attention mechanisms can be interpreted as explanations. Previous studies \cite{bibal2022attention,grimsley2020attention,jain2019attention,serrano2019attention} have shown that attention distributions may exhibit weak correlation with attribution measures and that substantially different attention patterns can lead to similar predictions. More generally, several attention-based interpretability methods have been proposed for Transformers, including relevance-propagation approaches that combine attention with gradient information \cite{chefer2021transformer}. Consequently, attention weights should not generally be regarded as faithful feature attributions or causal explanations.

The viewpoint adopted here is fully consistent with these observations. We do not interpret attention as explaining model predictions. Instead, rollout is viewed as an attention-derived propagation surrogate describing how attention mass is redistributed across tokens when layerwise attention matrices are composed. From this perspective, attribution-based methods and propagation-based analyses address complementary aspects of interpretability. 

\section{Attention-based interpretability}\label{sec:attention}

Attention rollout defines feature-to-feature propagation operators summarizing attention-mediated exchange across Transformer layers. Understanding the structural properties of these operators allows us to describe the organization encoded by rollout and motivates the rollout propagation scores introduced later in the paper.

\subsection{Transformers and the self-attention mechanism}\label{subsec:Transformer}

In this work, we consider a single-output Transformer-based regression model of the form
\begin{align*}
    x \in \RR^d \mapsto \xi(x) \in \RR^p \mapsto f(x)=w^\top \xi(x)+b \in\RR,
\end{align*}
where $p$ indicates the dimension of the latent representation space. At layer $\ell$, the hidden state is $Z^{(\ell)} \in \RR^{d\times p}$, whose $i$-th row $Z_i^{(\ell)}\in\RR^p$ corresponds to the latent representation of feature $i$. The latent representation $\xi(x)\in\RR^p$ is obtained by mean pooling over the feature representations of the last Transformer layer.

The central mechanism governing feature interactions is \emph{multi-head self-attention} \cite{vaswani2017attention}. For each layer $\ell\in\inter{L}$ and attention head $h\in\inter{H}$, the attention matrix is
\begin{align*}
	A^{(\ell,h)} = \text{softmax}\left(\frac{Q^{(\ell,h)} (K^{(\ell,h)})^\top}{\sqrt{p_h}}\right)\in\RR^{d\times d},
\end{align*}
where $p_h$ denotes the query/key dimension of head $h$. The matrices are constructed from
\begin{align*}
    \text{queries: } Q^{(\ell,h)} = Z^{(\ell-1)} W_Q^{(\ell,h)}, \quad \text{keys: } K^{(\ell,h)} = Z^{(\ell-1)} W_K^{(\ell,h)}.     
\end{align*}
Here, the notation
\begin{align*}
	\inter{q}\coloneqq\{1,\ldots,q\} \quad\text{ for all } q\in\NN^\ast,
\end{align*}
indicates the set containing the first $q$ strictly positive natural numbers. 

Notice that $A^{(\ell,h)}$, as all other attention-based quantities considered in this work, is actually dependent on the input datum $x\in\RR^d$, that is 
\begin{align*}
    A^{(\ell,h)}=A^{(\ell,h)}(x).    
\end{align*}

However, for simplicity of notation, we will avoid writing explicitly this dependence unless strictly necessary. Notice also that $A^{(\ell,h)}$ is row-stochastic as
\begin{align*}
	A^{(\ell,h)}_{ij}\geq 0 \;\;\text{ for all } i,j\in\inter{d} \quad\text{ and }\quad \displaystyle \sum_{j=1}^d A^{(\ell,h)}_{ij}=1 \;\;\text{ for all } i\in\inter{d}.    
\end{align*}

Moreover, to describe the average propagation geometry induced by multi-head attention, we introduce the head-averaged attention matrix
\begin{align*}
    A^{(\ell)} = \frac 1H\sum_{h=1}^H A^{(\ell,h)},
\end{align*}
that by construction remains row-stochastic. 

Finally, motivated by the presence of residual connections, we shall view the hidden-state evolution schematically through the effective propagation operators $I+A^{(\ell)}$, so that successive Transformer layers induce an iterative process of the form
\begin{align*}
    Z^{(L)} \rightsquigarrow \left(I+A^{(L)}\right) \cdots \left(I+A^{(1)}\right)Z^{(0)}.
\end{align*}

This effective description provides a convenient representation of how local feature interactions accumulate through depth and organize the global propagation dynamics inside the network.

\begin{remark}\label{rem:transformer}
While self-attention provides the principal mechanism through which information is exchanged across features, the standard Transformer block contains more components: the attention layer is followed by residual connections, normalization layers, and multilayer perceptrons (MLPs), all of which introduce additional nonlinear transformations of the hidden representations. The normalization layers rescale and stabilize the representations, whereas the MLP sublayers enrich the feature representations through nonlinear mixing in the latent space. Consequently, the exact evolution of the hidden states cannot be described solely in terms of attention matrices. Nevertheless, because self-attention is the component that explicitly governs feature-to-feature interactions, it provides a natural and tractable framework for interior interpretability. 
\end{remark}

\subsection{Attention rollout}\label{subsec:rollout} 

Attention rollout \cite{abnar2020quantifying} provides a global description of how local attention interactions accumulate across the sequence of attention layers.

Given an input datum $x\in\RR^d$, we define the \emph{local rollout operator} as the ordered composition
\begin{align}\label{eq:rollout}
	R(x) = \*A^{(L)}(x)\*A^{(L-1)}(x)\,\cdots\, \*A^{(1)}(x)\in\RR^{d\times d},    
\end{align}
with 
\begin{align*}
    \*A^{(\ell)}(x) = \frac 12\Big(I+A^{(\ell)}(x)\Big).  
\end{align*}

Since each $\*A^{(\ell)}$ is row-stochastic, the rollout matrix $R(x)$ is also row-stochastic for all $x\in\RR^d$. Its coefficient $R_{ij}(x)$ quantifies the cumulative attention-mediated propagation from source feature $j$ to destination feature $i$ under the rollout approximation.

To obtain a cumulative description of these feature interactions at the level of a dataset $\{x_n\}_{n=1}^N\subset\RR^d$, we also define the corresponding \emph{global rollout operator}
\begin{align}\label{eq:rollout_global}
	\*R = \frac 1N\sum_{n=1}^N R(x_n)\in\RR^{d\times d},    
\end{align}
which provides a coarse-grained description of the dataset-averaged geometry encoded by rollout. Notice that also $\*R$ is row-stochastic by construction.

Throughout this paper, attention rollout is interpreted as an attention-mediated propagation operator. Accordingly, expressions such as propagation, information flow, and influence always refer to the propagation encoded by the rollout operator, rather than to the exact evolution of the Transformer hidden states. As discussed in Remark \ref{rem:transformer}, the latter also depends on value projections, output projections, normalization layers, and MLP transformations, which are not represented by the rollout operator.

Beyond providing a compact description of attention-mediated propagation, rollout raises a natural mathematical question. For each input $x$, the local rollout $R(x)$ is an ordered product of row-stochastic propagation matrices, whereas the global rollout $\*R$ is the dataset average of these local products. The structural result below applies to either operator because both are row-stochastic. The depth-dependent product estimate developed in Appendix \ref{app:proof} is first established for the local rollout and then transferred to the global rollout by convexity. To state the structural result, we recall the Dobrushin contraction coefficient (see \cite{dobrushin1956central} or \cite[Chapter 4, Definition 4.6]{seneta2006non}), which quantifies the contraction of row-wise oscillations induced by a row-stochastic matrix.

\begin{definition}\label{def:dobrushin}
Given a row-stochastic matrix $P\in\RR^{d\times d}$, the Dobrushin contraction coefficient is defined by
\begin{align*}
    \kappa(P) \coloneqq 1-\min_{i,k\in\inter{d}} \sum_{j=1}^d \min\{P_{ij},P_{kj}\}.
\end{align*}
Equivalently,
\begin{align}\label{eq:dobrushin_def_eq}
    \kappa(P) = \frac12 \max_{i,k\in\inter{d}} \|p_i-p_k\|_1,    
\end{align}
where $p_i$ denotes the $i$-th row of $P$. 
\end{definition}

\begin{theorem}\label{thm:rollout}
Let $\MC R=(\MC R_{ij})_{i,j=1}^d\in\RR^{d\times d}$ be a row-stochastic rollout operator, and let $\kappa(\MC R)$ denote its Dobrushin coefficient. Define $v\coloneqq d^{-1}\MC R^\top\*1$, where $\*1$ denotes the all-ones vector in $\RR^d$. Then $v$ is a probability vector, the matrix $\Pi_{\MC R}\coloneqq\*1v^\top$ is row-stochastic and has rank one and, if we define
\begin{align*}
    \|\MC R-\Pi_{\MC R}\|_{\infty,1} \coloneqq \max_{i\in\inter{d}}\sum_{j=1}^d|(\MC R-\Pi_{\MC R})_{ij}|, 
\end{align*}
we have
\begin{align*}
    \kappa(\MC R)\leq \|\MC R-\Pi_{\MC R}\|_{\infty,1} \leq 2\left(\frac{d-1}{d}\right)\kappa(\MC R).    
\end{align*}
\end{theorem}

Notice that Theorem \ref{thm:rollout} is stated for a generic rollout operator $\MC R$, allowing it to be applied both to local rollout matrices $R(x)$ and to the global rollout $\*R$ defined in \eqref{eq:rollout_global}. 

Theorem \ref{thm:rollout} shows that, when the Dobrushin coefficient is small, the rollout operator $\MC R$ is well approximated by $\Pi_{\MC R} = \*1 v^\top$, with $v= d^{-1}\MC R^\top\*1$. Consequently, every row of $\MC R$ is close to the same probability vector $v$. More precisely, for every destination index $i\in\inter{d}$,
\begin{align*}
    \sum_{j=1}^d \left|\MC R_{ij}-v_j\right|
    \leq 2\left(\frac{d-1}{d}\right)\kappa(\MC R).
\end{align*}

Thus, when $\kappa(\MC R)$ is small, the distribution of source-feature contributions received by the different destination features depends only weakly on the destination index. Since
\begin{align*}
    \sum_{i=1}^d \mathcal R_{ij}=dv_j,  
\end{align*}
the coefficient $v_j$ is exactly the normalized column mass of source feature $j$. Thus the common row profile $v$ identifies how propagation is distributed among the source features, while the dependence on the destination feature is approximately lost. 

The contraction estimate used in Theorem \ref{thm:rollout} is a direct finite-dimensional consequence of classical results on stochastic matrices and ergodicity coefficients developed in the theory of nonhomogeneous Markov chains; see \cite{dobrushin1956central} and \cite[Section 4.3]{seneta2006non}. Related contraction phenomena have also appeared in mathematical analyses of deep attention mechanisms \cite{alcalde2025clustering,alvarez2026perceptrons,dong2021attention,geshkovski2023emergence,geshkovski2025mathematical}. Its role here is to provide an analytical lens for attention rollout: it identifies the normalized column sums as the mean row profile and quantifies when the destination-dependent variation around that profile is small.

\begin{remark}
No novelty is claimed for the Doeblin--Dobrushin contraction estimate itself. The contribution made here is its application and interpretation in the attention-rollout setting, including the distinction between contraction of the rows and heterogeneity of the resulting propagation profile. This distinction motivates the propagation scores used in the empirical analysis.
\end{remark}

\subsection{Rollout propagation scores}\label{subsec:rollout_importance}

Theorem \ref{thm:rollout} identifies the normalized column-sum vector $v$ as the common row profile appearing in the rank-one approximation of the rollout operator. This gives the column sums a precise structural interpretation: when the Dobrushin coefficient is small, they encode the remaining dependence of the common propagation profile on the features after the row-wise variability has contracted. Motivated by this observation, we use the column sums as propagation scores.

Given the global rollout matrix $\*R\in\RR^{d\times d}$ introduced in \eqref{eq:rollout_global}, we therefore define the rollout column-score
\begin{align}\label{eq:feature_score}
	s_j = \sum_{i=1}^d \*R_{ij} \quad\text{ for all }j\in\inter{d},
\end{align}
which measures the total propagated mass associated with source feature $j$, aggregated over all destination features. Features with larger $s_j$ therefore contribute more strongly to the global propagation profile encoded by the rollout operator. Consequently, sorting the features according to $s_j$ yields a propagation-based ranking describing their relative contribution to that profile. Notice that, since $s_j=dv_j$, where $v$ is the probability profile introduced in Theorem \ref{thm:rollout}, the scores $s_j$ differ from $v_j$ only by the constant normalization factor $d$ and therefore induce exactly the same ranking. Theorem \ref{thm:rollout} provides the structural motivation for these scores. Because rollout omits value and output projections, normalization, and MLP transformations, the scores should not be interpreted as causal effects or as faithful attributions of the complete predictor.

\section{Experimental setup and interpretability pipeline}\label{sec:methods}

In this section, we describe the experimental framework used to examine the proposed propagation analysis, including the datasets, the models, and the analysis pipeline.

\subsection{Dataset}\label{subsec:data}

For our experiments, we considered one biomedical dataset and one language-model-generated tabular dataset for age prediction. Each sample is described by a set of metabolic, biochemical, clinical, and possibly environmental variables. Their main characteristics are summarized in Table \ref{tab:datasets}.

\texttt{Dataset 1} consists of anonymized metabolomic measurements obtained from an existing biomedical cohort under a data-sharing agreement with the data owners. Each participant contributed a single observation. The data are not publicly available because of ethical and legal restrictions. The original cohort design, recruitment procedures, laboratory protocols, and ethical approvals are described in \cite{ibanez2025metabolomic}, which should be consulted for clinical details beyond the scope of the present methodological study.

\texttt{Dataset 2}, instead, consists of synthetic tabular data generated using ChatGPT. The generation prompt requested plausible ranges and qualitative relationships among common metabolic, biochemical, and clinical variables used in age-prediction studies. The dataset was not intended to reproduce the statistical distribution of any specific cohort, provide known ground-truth feature effects, or serve as a benchmark dataset. It is used only as a complementary qualitative example. The complete dataset, including the original \texttt{csv dataframe}, is publicly available in the project repository \cite{Github} and is sufficient to reproduce the reported experiments. 

Prior to model training, observations with missing values in any of the input variables were excluded, resulting in a complete-case dataset. The remaining observations were randomly divided into training, validation, and test subsets in proportions $80\%$, $10\%$, and $10\%$, respectively. All subsequent data-dependent preprocessing decisions were made exclusively on the training subset. To reduce redundancy among the input variables, we computed pairwise Pearson correlations on the training data. For every pair of variables whose absolute correlation was at least $0.85$, one variable was removed. The resulting feature set and column order were then fixed and applied unchanged to the validation and test subsets. Finally, the mean and standard deviation used to normalize each input variable were estimated from the training subset only and applied without refitting to the validation and test subsets. The complete preprocessing pipeline is available on GitHub \cite{Github}.

\renewcommand{\arraystretch}{1.3}
\begin{SCtable}[2][!h]
	\centering
	\caption{Summary of the datasets used in the experimental study. For each dataset, we report the number of samples in the train/validation/test split and the number of input features before and after the preprocessing pipeline.}
	\label{tab:datasets}
	
	\begin{tabular}{|l|c|c|}
		\hline
		& \cellcolor{lightgray}\textbf{Dataset 1} &
		\cellcolor{lightgray}\textbf{Dataset 2} \\
		\hline
		\cellcolor{lightgray}\textbf{\# train data} & 12744 & 24000 \\
		\hline
		\cellcolor{lightgray}\textbf{\# validation data} & 1593 & 3000 \\
		\hline
		\cellcolor{lightgray}\textbf{\# test data} & 1594 & 3000 \\
		\hline
		\cellcolor{lightgray}\textbf{\# features before processing} & 75 & 20 \\
		\hline
		\cellcolor{lightgray}\textbf{\# features after processing} & 72 & 18 \\
		\hline
	\end{tabular}
	
\end{SCtable}

\subsection{Model and training setup} 

All the models considered in our experiments share the same general structure: an embedding layer followed by a stack of Transformer encoder blocks and a final linear regression head to generate the predictions. Each scalar input feature $x_j$ is mapped to a token through a feature-specific affine embedding $x_j w_j+e_j$, where $w_j,e_j\in\RR^p$ are learned parameters associated with feature $j$. The embedding dimension is fixed to $p=64$ and each attention block employs $4$ attention heads of dimension $p_h=16$. The feedforward sublayers have hidden dimension $128$ and use $0.1$ dropout. No positional encoding is used because the metabolomic variables have no canonical sequential ordering. Feature identity is nevertheless encoded by the feature-specific parameters $w_j$ and $e_j$; thus, the architecture does not impose a sequence order but is not permutation-invariant with respect to arbitrary reassignment of feature identities.

The main parameter varying across architectures is the number of self-attention layers. More precisely, we train Transformer models with progressively increasing depth, ranging from shallow to deeper configurations. This family of models allows us to study how the attention mechanisms evolve as feature representations propagate through longer sequences of interaction operators. As shown in Section \ref{sec:experiments}, deeper models in the present experiments exhibit increasing row homogenization of the rollout operator.

Training is performed using the Adam optimizer with learning rate $10^{-3}$, weight decay $10^{-5}$, and batch size $128$, using early stopping and \texttt{ReduceLROnPlateau}. The loss function employed is \texttt{SmoothL1Loss}.

Predictive performance is evaluated on the test split using standard regression metrics: Root Mean Squared Error (RMSE), Mean Absolute Error (MAE), and $\text{R}^2$. Since the target variable is normalized during training, RMSE and MAE are computed after denormalizing the predictions and ground-truth labels and are therefore reported in units of years.

Simulations were performed on an Acer TravelMate P414-53 laptop with OS Ubuntu 24.04.3 LTS, 13th Gen Intel\textsuperscript{\textregistered} Core\textsuperscript{\texttrademark} i5-1335U × 12 processor and 32 GiB of RAM memory. Moreover, all experiments were repeated using three independent random seeds. Unless otherwise stated, quantitative performance metrics are reported as mean and standard deviation across the three independent training runs.

\subsection{Interpretability methods}\label{subsec:interpret_methods}

The objective of our empirical analysis is to compare the three viewpoints introduced above. PCA supplies a variance-based score, GradientExplainer supplies an approximate prediction-attribution score, and rollout supplies a propagation score.

\medskip 
\noindent\textbf{Variance viewpoint: PCA}-\textbf{based feature importance.}
As an unsupervised reference, we perform PCA on the normalized data. This provides a description of the intrinsic variance structure of the metabolomic feature space independently of the predictor. Feature importance is quantified through the singular-value-weighted loading score
\begin{align}\label{eq:pca_score_data}
    r_j = \sum_{k=1}^d \omega_k \vert \zeta_{kj}\vert, \quad\text{ for all } j\in\inter{d}, 
\end{align}
which measures the overall contribution of feature $j$ to the principal-component representation of the data, weighting the contribution of each principal component by its singular value. Here,
\begin{displaymath}
    \begin{array}{l}
        \omega_k\in\RR 
        \\
        \zeta_k = (\zeta_{kj})_{j=1}^d\in\RR^d
    \end{array}, \quad \text{ for all } k\in\inter{d}
\end{displaymath}
are the singular values and right singular vectors in the SVD decomposition of the data matrix.

\medskip 
\noindent\textbf{Propagation viewpoint: rollout propagation scores.}
We analyze feature interactions through the global rollout matrix $\*R$ constructed as in \eqref{eq:rollout_global}. Besides computing the rollout propagation scores \eqref{eq:feature_score}, we also quantify the structural organization of the propagation operator by evaluating its Dobrushin coefficient $\kappa(\*R)$ and the approximation error $\|\*R-\Pi_{\*R}\|_{\infty,1}$ appearing in Theorem \ref{thm:rollout}. These quantities measure the extent to which the rollout operator is close to a rank-one propagation profile at each Transformer depth.

\medskip
\noindent\textbf{Attribution viewpoint: GradientExplainer scores.}
We estimate feature-attribution scores with the \texttt{GradientExplainer} implementation of the Python \texttt{SHAP} package \cite{lundberg2017unified}. This implementation computes expected-gradients attributions relative to an empirical background distribution and is used as an approximation to SHAP values for differentiable models. These numerical scores need not coincide exactly with either interventional or conditional Shapley values. The exact interventional Shapley operator used in Proposition \ref{prop:shap} is defined separately in Appendix \ref{app:proof}; that idealized proposition is not a validation theorem for the empirical \texttt{GradientExplainer} scores. The background supplied to \texttt{GradientExplainer} consists of a random subset of $256$ samples from the test set. Methodologically, using the test set as the \texttt{GradientExplainer} background does not introduce information leakage, because the background distribution is used exclusively after the models have been fully trained to compute post hoc SHAP values. It is not involved in training, hyperparameter tuning, model selection, or performance evaluation, and therefore cannot influence the learned models or their predictive performance. Moreover, for our dataset, the choice of training or test background has negligible impact. We verified that the training and test sets are statistically indistinguishable by training a Random Forest classifier to discriminate between them. The classifier achieved an AUC of $0.495 \pm 0.020$, i.e., no better than random guessing, indicating that the two sets are effectively drawn from the same distribution. 

\section{Experimental results}\label{sec:experiments}

We first analyze the real metabolomic dataset (\texttt{Dataset 1}), which serves as the main case study throughout the paper. We then repeat selected analyses on the synthetic \texttt{Dataset 2} as a complementary qualitative example. All experiments presented in this section, together with the corresponding code, are available on \texttt{GitHub} \cite{Github}.

\subsection{Experimental results on Dataset 1}

We present here the results of our experiments performed on \texttt{Dataset 1}.

\subsubsection{Training and predictive performance}\label{subsec:training}

To investigate the effect of depth, we trained Transformer architectures with varying numbers of self-attention layers ranging from $3$ to $20$. The predictive performance obtained on the test set is reported in Table \ref{tab:metrics}.

\renewcommand{\arraystretch}{1.3}
\begin{table*}[!h]
    \centering
    \begin{tabular}{|c|c|c|c|c|}
        \hline \rowcolor{lightgray} \textbf{Depth} & $\text{\textbf{R}}^2$ & \textbf{RMSE} & \textbf{MAE} & \textbf{Best validation loss}
        \\
        \hline $3$ & $0.6942 \pm 0.0048$ & $7.5414 \pm 0.0597$ & $5.8941 \pm 0.0529$ & $0.1574 \pm 0.0017$
        \\
        \hline $5$ & $0.6980 \pm 0.0083$ & $7.4942 \pm 0.1035$ & $5.8787 \pm 0.0623$ & $0.1584 \pm 0.0009$
        \\
        \hline $7$ & $0.6867 \pm 0.0121$ & $7.6322 \pm 0.1474$ & $5.9504 \pm 0.0539$ & $0.1568 \pm 0.0016$
        \\
        \hline $9$ & $0.6866 \pm 0.0019$ & $7.6354 \pm 0.0230$ & $6.0094 \pm 0.0227$ & $0.1577 \pm 0.0012$
        \\
        \hline $11$ & $0.6919 \pm 0.0056$ & $7.5705 \pm 0.0686$ & $5.9114 \pm 0.0496$ & $0.1582 \pm 0.0007$
        \\
        \hline $13$ & $0.6872 \pm 0.0031$ & $7.6277 \pm 0.0380$ & $5.9542 \pm 0.0238$ & $0.1567 \pm 0.0021$
        \\
        \hline $15$ & $0.6861 \pm 0.0059$ & $7.6424 \pm 0.0714$ & $5.9406 \pm 0.0150$ & $0.1569 \pm 0.0032$
        \\
        \hline $17$ & $0.6910 \pm 0.0064$ & $7.5816 \pm 0.0787$ & $5.9594 \pm 0.0511$ & $0.1596 \pm 0.0025$
        \\
        \hline $20$ & $0.6891 \pm 0.0031$ & $7.6043 \pm 0.0377$ & $5.9613 \pm 0.0381$ & $0.1618 \pm 0.0015$             
        \\
        \hline
    \end{tabular}
    \vspace{6pt}
    \caption{Training and predictive performance over \texttt{Dataset 1} of Transformer architectures with increasing depth. The table reports predictive accuracy on the test split for Transformer models with different numbers of self-attention layers. Results are presented as the mean $\pm$ standard deviation over independent runs using different random seeds, summarizing the variability observed across initializations. The reported metrics include the coefficient of determination $\text{R}^2$, Root Mean Squared Error (RMSE), Mean Absolute Error (MAE), and the best validation loss achieved during training.} \label{tab:metrics}
\end{table*}

Overall, performance remains stable across the considered depths. The coefficient of determination stays in the approximate range $\text{R}^2\in [0.68,0.70]$, while RMSE and MAE vary only moderately.

Figure \ref{fig:losses} complements this analysis by reporting the training and validation losses for representative shallow, intermediate, and deep Transformer architectures. All displayed curves converge rapidly during the first epochs and do not show a large sustained separation between training and validation loss.

\begin{figure*}[!h]
	\centering
	\includegraphics[width=0.9\textwidth]{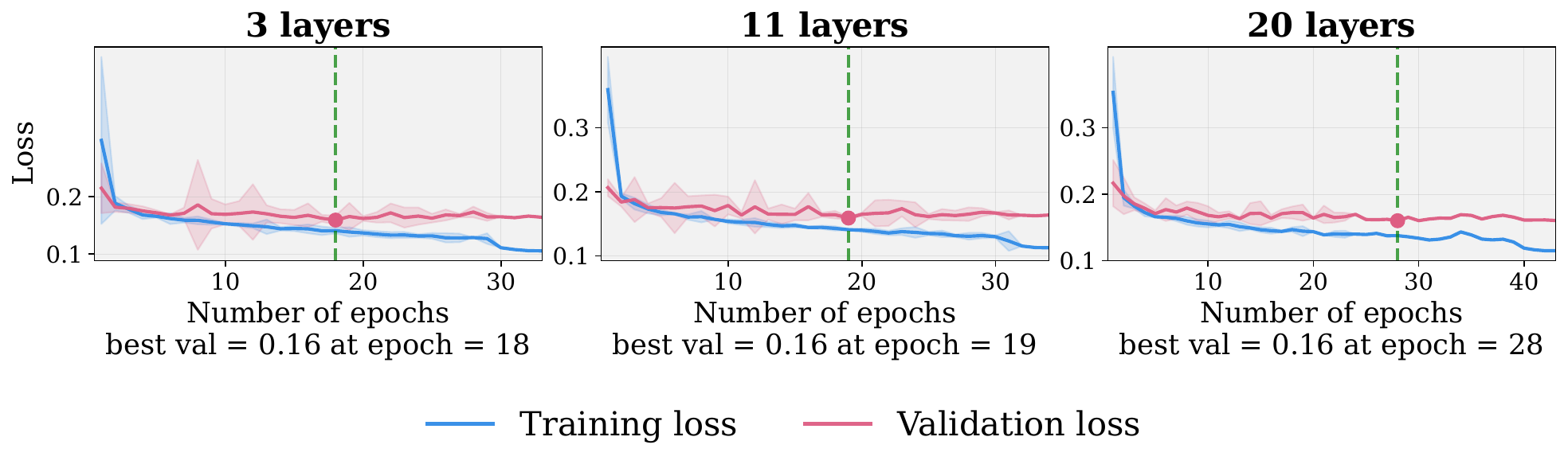}	
	\caption{Evolution of training and validation losses during optimization over \texttt{Dataset 1} for representative shallow (3-layer), intermediate (11-layer), and deep (20-layer) Transformer architectures. Shaded regions indicate $\pm 1$ standard deviation over three runs.}\label{fig:losses}
\end{figure*}

Across the depths considered, the observed differences in predictive performance are modest relative to the overall performance level, and no monotone improvement with depth is apparent. These results are descriptive and are not intended to establish statistical equivalence among architectures.

This observation motivates the analysis developed in Sections \ref{subsect:attn_iterpret}--\ref{subsec:feature_importance}: while the reported predictive performance remains broadly stable, the measured attention-propagation patterns continue to change with depth.

\subsubsection{Attention-based analysis}\label{subsect:attn_iterpret}

Motivated by the structural characterization provided by Theorem \ref{thm:rollout} and by the depth-dependent contraction estimate \eqref{eq:dobrushin_R_est_global} proved in Appendix \ref{app:proof}, we now examine how the corresponding column-wise organization is manifested in the trained Transformers.

Figure \ref{fig:attn_rollout} displays the global rollout operator together with its rank-one approximation and the associated error. Experimentally, as the depth increases, the rollout operator becomes increasingly well described by a common propagation profile. For each depth, the corresponding approximation is quantitatively characterized by Theorem \ref{thm:rollout} through the measured Dobrushin coefficient. The learned propagation profile remains nonuniform, showing empirically that some source features contribute more strongly to the propagation dynamics than others. Notice that Theorem \ref{thm:rollout} does not predict this nonuniformity; it is an empirical property of the learned rollout operator.

\begin{figure*}[!h]
	\centering	
    \includegraphics[width=0.9\textwidth]{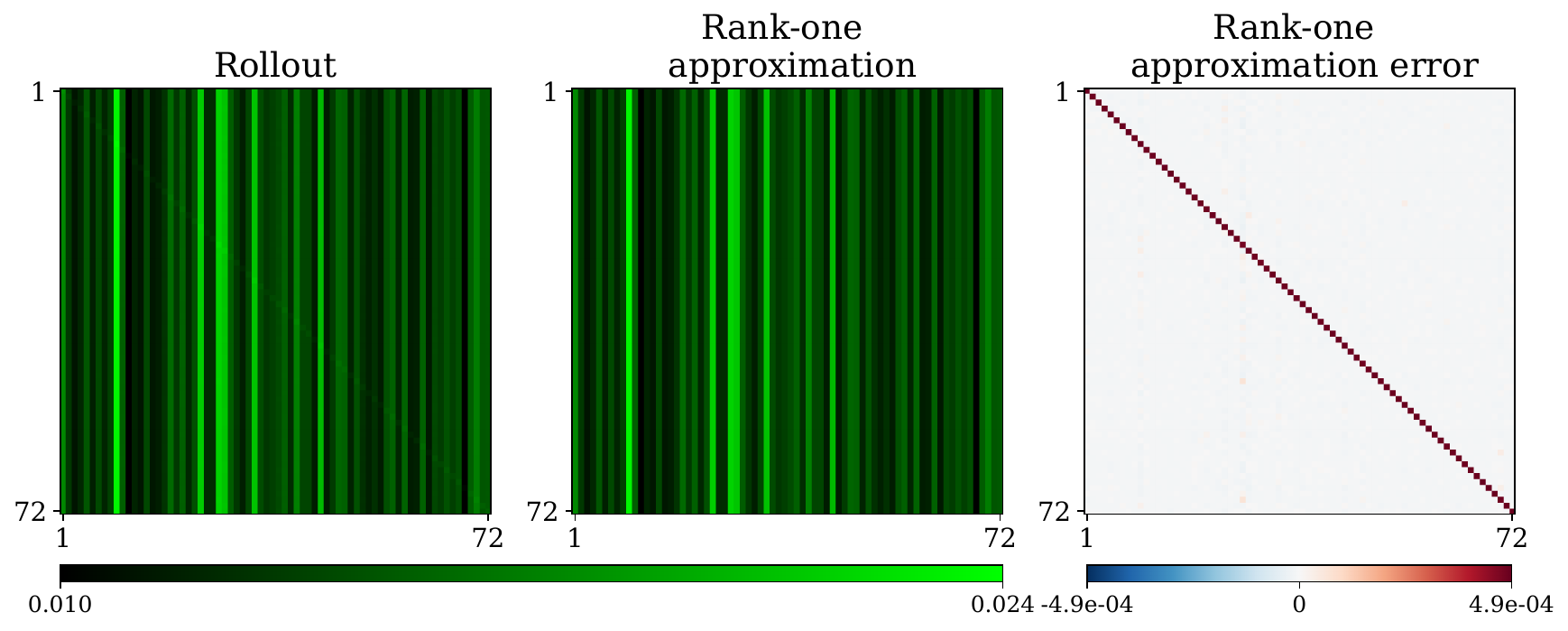}    
	\caption{Global rollout matrix \eqref{eq:rollout_global} computed over the test split of \texttt{Dataset 1} for one representative trained 11-layer Transformer realization, together with its rank-one approximation and the corresponding approximation error. The small residual error illustrates that the rollout operator is well approximated by a rank-one stochastic matrix, in accordance with Theorem \ref{thm:rollout}.}\label{fig:attn_rollout}
\end{figure*}

Figure \ref{fig:dobrushin} reports the Dobrushin contraction coefficient $\kappa(\*R)$ together with the approximation error $\|\*R-\Pi_{\*R}\|_{\infty,1}$. By Theorem \ref{thm:rollout}, these two quantities are analytically equivalent up to fixed multiplicative constants for every rollout operator. Moreover, estimate \eqref{eq:dobrushin_R_est_global} controls the Dobrushin coefficient of the global rollout, making explicit the factors associated with residual averaging and the layerwise attention matrices. In the architectures considered here, both the Dobrushin coefficient and the approximation error are observed experimentally to decrease rapidly with depth.

\begin{figure}[t]
	\centering
	
	\begin{minipage}[c]{0.47\columnwidth}
		\centering
		\includegraphics[width=\linewidth]{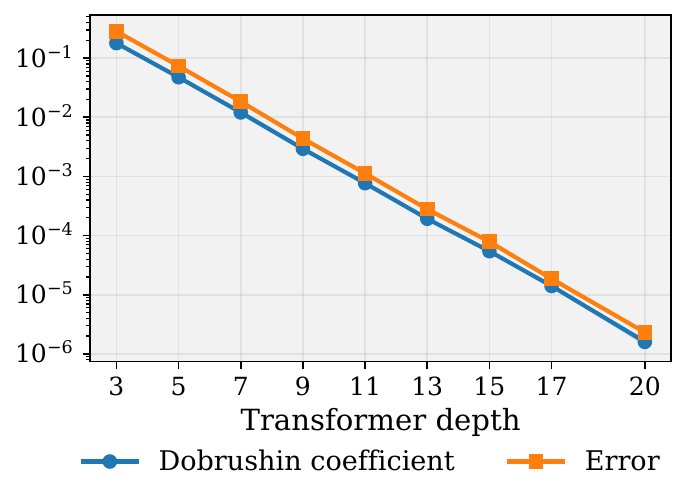}
	\end{minipage}
	\hfill
	\begin{minipage}[c]{0.50\columnwidth}
		\centering
		\vspace{-1.2cm}
		{\setlength{\tabcolsep}{3pt}
			\renewcommand{\arraystretch}{1.1}
			\scriptsize
			        \begin{tabular}{|c|c|c|}
				\hline
				\rowcolor{lightgray}\textbf{Depth} & \textbf{Dobrushin coefficient $\kappa(\*R)$} & \textbf{Error} $\|\*R-\Pi_{\*R}\|_{\infty,1}$ 
				\\
				\hline\hline $3$ & $1.7\times 10^{-1} \pm 5.5 \times 10^{-5}$ & $2.8\times10^{-1} \pm 5.9 \times 10^{-3}$ 
				\\
				$5$ & $4.7\times 10^{-2} \pm 3.5 \times 10^{-3} $ & $7.3\times10^{-2} \pm 1.4 \times 10^{-3}$ 
				\\
				$7$ & $1.2\times 10^{-2} \pm 1.2 \times 10^{-3}$ & $1.8\times10^{-2} \pm 1.6 \times 10^{-3}$ 
				\\
				$9$ & $2.9\times 10^{-3} \pm 2.5 \times 10^{-4}$ & $4.3\times10^{-3} \pm 1.9 \times 10^{-4}$ 
				\\
				$11$ & $7.6\times 10^{-4} \pm 1.2 \times 10^{-5}$ & $1.1\times10^{-3} \pm 1.1 \times 10^{-5}$ 
				\\
				$13$ & $1.9\times 10^{-4} \pm 1.1 \times 10^{-5}$ & $2.7\times10^{-4} \pm 3.8 \times 10^{-6}$ 
				\\
				$15$ & $5.4\times 10^{-5} \pm 3.1 \times 10^{-6}$ & $7.8\times10^{-5} \pm 4.6 \times 10^{-7}$ 
				\\    
				$17$ & $1.4\times 10^{-5} \pm 7.1 \times 10^{-7}$ & $1.8\times10^{-5} \pm 2.3 \times 10^{-7}$ 
				\\
				$20$ & $1.5\times 10^{-6} \pm 2.2 \times 10^{-7}$ & $2.3\times10^{-6}\pm 1.2 \times 10^{-7}$ 
				\\
				\hline
			\end{tabular}        
		}
	\end{minipage}
	
	\caption{Decay of the Dobrushin contraction coefficient $\kappa(\*R)$ and the approximation error $\|\*R-\Pi_{\*R}\|_{\infty,1}$ as functions of the Transformer depth on \texttt{Dataset 1}. The table reports the corresponding mean $\pm$ standard deviation over independent training runs. Theorem \ref{thm:rollout} guarantees that these two quantities remain quantitatively equivalent up to fixed multiplicative constants for every rollout operator.}\label{fig:dobrushin}
\end{figure}

The observed decrease in the trained models is compatible with the product mechanism identified in Appendix \ref{app:proof}, but the present measurements do not quantify the separate contribution of training to contraction. Section \ref{subsec:untrain} therefore addresses a narrower empirical question by comparing the propagation profiles of trained and randomly initialized Transformers.

\subsubsection{Effect of learning on the propagation profile}\label{subsec:untrain}

The previous section documented depth-dependent contraction for the trained Transformers. We now address a separate question: whether trained and randomly initialized models exhibit different propagation profiles. The comparison below concerns the profiles and their rankings; it does not compare their Dobrushin coefficients.

Our analysis is based on two complementary experiments. First, we compare the propagation profiles $v$ obtained from trained and randomly initialized (untrained) models using their $\ell^1$ distance. Second, we measure the $\ell^1$ distance of each profile from the uniform distribution $u=d^{-1}\*1$. Together, these descriptive quantities assess how the distribution of rollout mass differs between the two groups.

The results are reported in Tables \ref{tab:l1_profile_distance} and \ref{tab:uniform_profile_distance}. Table \ref{tab:l1_profile_distance} shows that the trained and randomly initialized profiles differ across all considered depths, with an $\ell^1$ distance of approximately $0.15$ in these experiments. This comparison does not by itself establish the cause of the difference or whether the resulting ranking is predictively relevant.

Table \ref{tab:uniform_profile_distance} compares each propagation profile with the uniform distribution. The reported randomly initialized profiles remain close to uniform, whereas the trained profiles have larger distances from uniformity. This descriptive comparison indicates more heterogeneous rollout profiles in the trained models considered here.

\begin{table}[ht]
\centering

\begin{minipage}[t]{0.44\textwidth}
\centering
\renewcommand{\arraystretch}{1.15}
\begin{tabular}{|c|c|}
    \hline
    \rowcolor{lightgray}\textbf{Depth} & \textbf{$\ell^1$ distance} \\
    \hline
    3  & $0.1438 \pm 0.0006$ \\
    5  & $0.1644 \pm 0.0280$ \\
    7  & $0.1529 \pm 0.0212$ \\
    9  & $0.1452 \pm 0.0181$ \\
    11 & $0.1511 \pm 0.0038$ \\
    13 & $0.1489 \pm 0.0062$ \\
    15 & $0.1618 \pm 0.0164$ \\
    17 & $0.1573 \pm 0.0114$ \\
    20 & $0.1318 \pm 0.0087$ \\
    \hline
\end{tabular}

\captionof{table}{$\ell^1$ distance between the propagation profiles $v$ of trained and randomly initialized Transformer models as a function of the depth. Results are reported as mean $\pm$ standard deviation over independent random seeds.}
\label{tab:l1_profile_distance}
\end{minipage}
\hfill
\begin{minipage}[t]{0.53\textwidth}
\centering
\renewcommand{\arraystretch}{1.15}
\begin{tabular}{|c|c|c|}
    \hline
    \rowcolor{lightgray}\textbf{Depth} & \textbf{Trained} & \textbf{Random init.} \\
    \hline
    3  & $0.1434 \pm 0.0015$ & $0.018273 \pm 0.001544$ \\
    5  & $0.1616 \pm 0.0308$ & $0.018319 \pm 0.001459$ \\
    7  & $0.1506 \pm 0.0201$ & $0.018258 \pm 0.001479$ \\
    9  & $0.1428 \pm 0.0188$ & $0.018272 \pm 0.001484$ \\
    11 & $0.1511 \pm 0.0022$ & $0.018272 \pm 0.001487$ \\
    13 & $0.1469 \pm 0.0063$ & $0.018272 \pm 0.001487$ \\
    15 & $0.1606 \pm 0.0168$ & $0.018272 \pm 0.001487$ \\
    17 & $0.1553 \pm 0.0097$ & $0.018272 \pm 0.001487$ \\
    20 & $0.1314 \pm 0.0073$ & $0.018272 \pm 0.001487$ \\
    \hline
\end{tabular}

\captionof{table}{$\ell^1$ distance between the propagation profile and the uniform distribution for trained and randomly initialized Transformer models as a function of the depth. Results are reported as mean $\pm$ standard deviation over independent random seeds.}
\label{tab:uniform_profile_distance}
\end{minipage}
\end{table}

To illustrate how the profile differences appear in the rollout ranking, Figure \ref{fig:untrained_rollout} compares the six highest-ranked variables obtained from trained and randomly initialized Transformers for the representative 11-layer architecture. We stress that this comparison is intended only as a visual illustration of the quantitative differences reported in Tables \ref{tab:l1_profile_distance} and \ref{tab:uniform_profile_distance}.

\begin{SCfigure}[0.5][ht]
	\centering
	\includegraphics[width=0.65\linewidth]{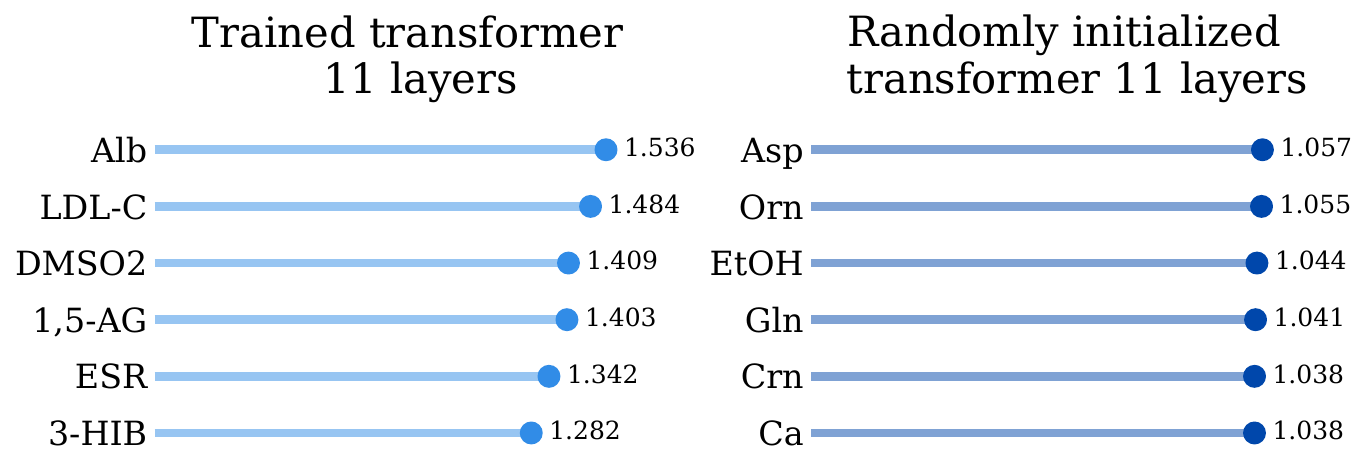}
	
	\caption{Comparison of rollout propagation scores over \texttt{Dataset 1} between trained and randomly initialized Transformers with 11 attention layers.}
	\label{fig:untrained_rollout}
\end{SCfigure}

Consistent with the profile comparison, the trained and randomly initialized models produce different rollout rankings. Some variables highly ranked after training, including albumin and LDL cholesterol, are established clinical variables discussed in Appendix \ref{app:bio}. This provides a qualitative plausibility check only: it does not demonstrate that rollout scores are faithful, causal, or predictively relevant.

These results show that the propagation profile and its ranking differ between trained and randomly initialized models. Together with the product analysis of Section \ref{app:product}, they motivate a distinction between the structural contraction mechanism of rollout composition and the profile realized after optimization. Establishing a causal relation between training, the profile, and predictive behavior requires additional intervention-based experiments.

\subsubsection{Comparison of feature-ranking scores}\label{subsec:feature_importance}

Section \ref{subsec:untrain} documented differences between the propagation profiles and rankings of the trained and randomly initialized models. We now compare three complementary scoring systems for the trained Transformer:
\begin{itemize}
    \item the PCA scores defined in \eqref{eq:pca_score_data}, obtained from the singular-value-weighted principal-component representation of the data;
    \item the rollout column-sum scores defined in \eqref{eq:feature_score}, measuring the rollout mass associated with each source feature;
    \item GradientExplainer scores, providing approximate prediction attributions relative to the selected background distribution.
\end{itemize}

Before comparing the different interpretability methods, we describe the cross-seed stability of the corresponding feature rankings. Figure \ref{fig:ranking_stability} reports the mean pairwise Spearman correlation of the rollout and GradientExplainer rankings across independent training runs.

\begin{SCfigure}[0.5][ht]
	\centering
	\includegraphics[width=0.65\columnwidth]{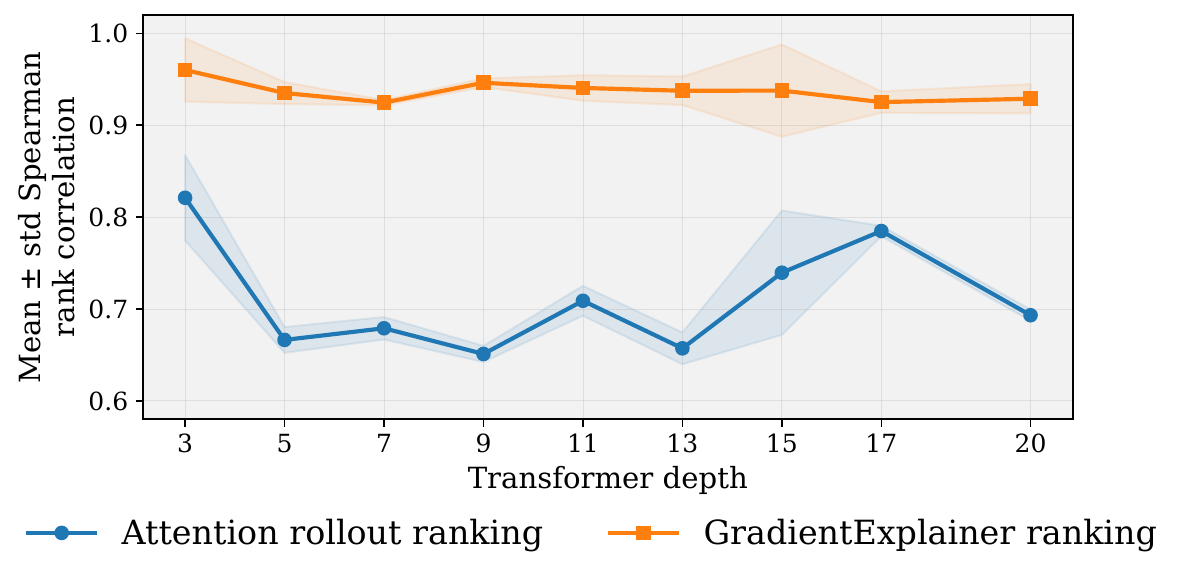}    
	\caption{Cross-seed stability of rollout- and GradientExplainer-based feature rankings over \texttt{Dataset 1}. For each Transformer depth, the figure reports the mean pairwise Spearman correlation between rankings obtained from independent training runs with different random seeds. The solid curves represent the mean across runs, and the shaded bands indicate one standard deviation.}\label{fig:ranking_stability}
\end{SCfigure}


Within the three reported initializations, the GradientExplainer rankings have higher mean pairwise Spearman correlations than the rollout rankings. The latter range from approximately $0.65$ to $0.80$. These values describe cross-seed stability in the present experiment; they do not establish population-level reproducibility.

Figure \ref{fig:importance} reports the six highest-ranked variables according to these scores for one representative trained Transformer of depth $11$. This intermediate depth is used throughout the qualitative illustrations because its measured Dobrushin coefficient is already small and its rollout operator already exhibits the rank-one propagation structure characterized by Theorem \ref{thm:rollout}. Table \ref{tab:features_name} reports the meaning of the feature abbreviations.

\begin{figure*}
	\centering
	\includegraphics[width=0.9\textwidth]{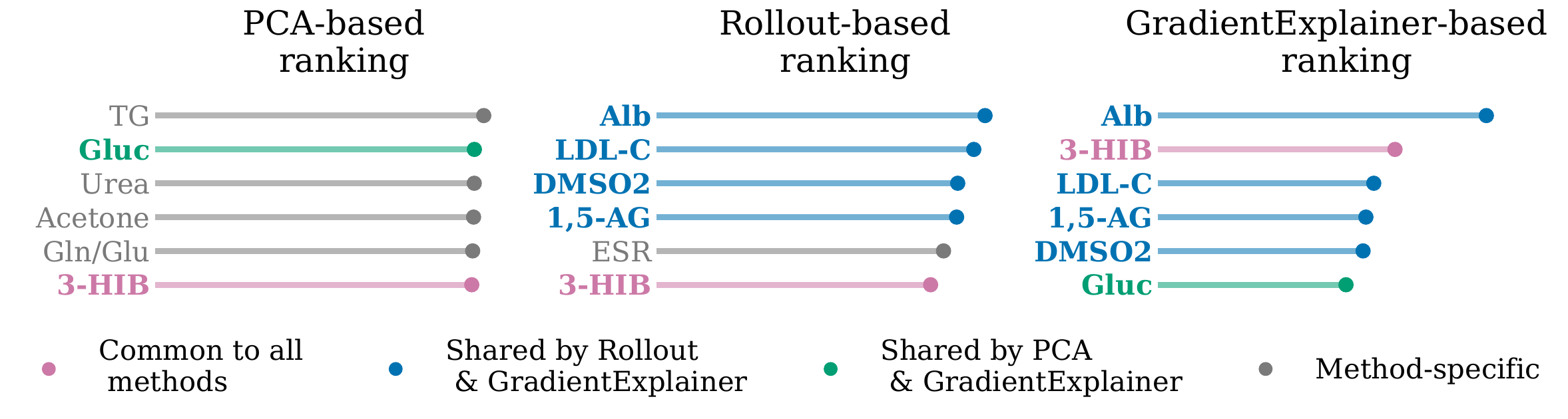}
	\caption{Comparison of top-$6$ rankings for \texttt{Dataset 1} obtained with PCA scores, rollout propagation scores, and GradientExplainer feature-attribution scores.}\label{fig:importance}
\end{figure*}

\begin{table*}[!h]
	\centering	
	\begin{tabular}{|c|c|c|} \hline
		\rowcolor{lightgray} \textbf{Feature} & \textbf{Abbreviation} & \textbf{Feature type} 
		\\
		\hline \hline 1,5-Anhydrosorbitol & 1,5-AG & Metabolite 
		\\
		3-Hydroxyisobutyric acid & 3-HIB & Metabolite 
        \\
		Acetone & Acetone & Metabolite 
        \\
		Albumin & Alb & Protein 
		\\
		Dimethylsulfone & DMSO2 & Metabolite 
        \\
		Erythrocyte Sedimentation Rate & ESR & Inflammatory marker 		 
        \\
		  Glucose & Gluc & Metabolite 
        \\
		  Glutamine to Glutamate ratio & Gln/Glu & Derived ratio 
		\\ 
		LDL Cholesterol & LDL-C & Clinical biomarker 
        \\
		Triglycerides & TG & Clinical biomarker 		
        \\
		Urea & Urea & Clinical biomarker 		
		\\
		\hline 
	\end{tabular}  
    \vspace{0.3cm}
	\caption{Correspondence between feature abbreviations and full variable names in Figure \ref{fig:importance}. Variables are grouped by category (metabolites, clinical biomarkers, proteins, inflammatory markers and derived ratios).} \label{tab:features_name}
\end{table*}

The PCA ranking differs substantially from both model-dependent rankings. The rollout and attribution rankings, by contrast, share several features. This difference is unsurprising because PCA ranks variables according to the variance structure of the input data alone and therefore has no access to the prediction task. This is also reflected in the nearly uniform PCA scores shown in Figure \ref{fig:importance}, where the highest-ranked variables receive very similar scores.

To compare the three interpretability methods quantitatively, Figure \ref{fig:agreement} summarizes their agreement across Transformer depths. 

The upper panel reports the normalized overlap between the six highest-ranked features identified by each pair of methods, averaged over three random initializations. The lower panel reports the Spearman correlation between the complete rankings, likewise averaged over the three initializations. Both panels are descriptive summaries of one cohort and a small number of training runs.

\begin{figure*}
	\centering
    \begin{minipage}[c]{0.49\textwidth}
        \centering
        \includegraphics[width=\linewidth]{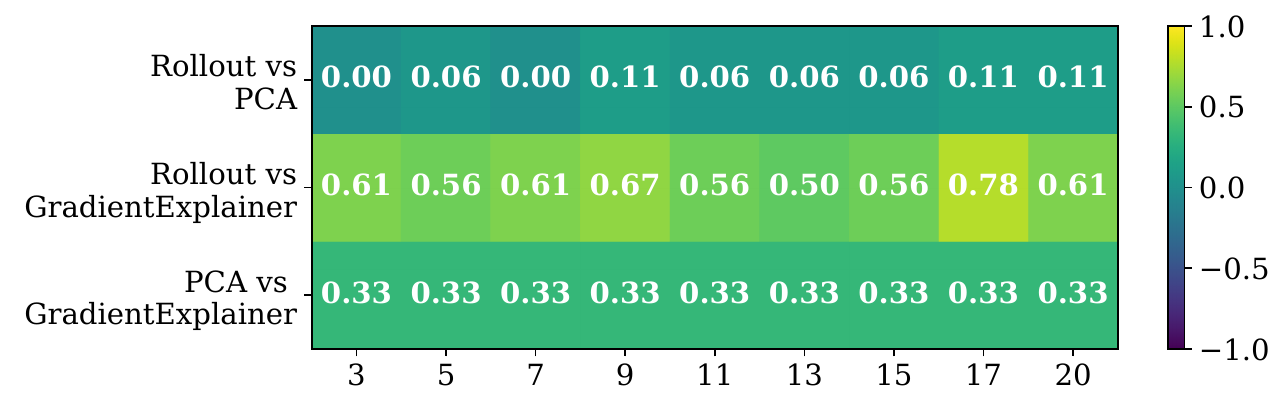}
    \end{minipage}
    \hfill
    \begin{minipage}[c]{0.49\textwidth}
        \centering
        \includegraphics[width=\linewidth]{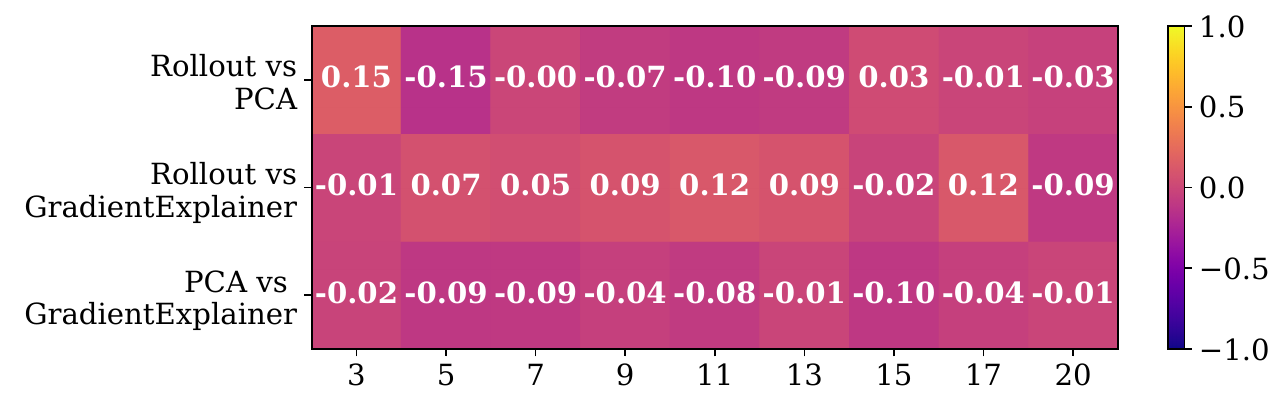}
    \end{minipage}
	\caption{Descriptive comparison of rankings obtained with PCA scores, rollout propagation scores, and GradientExplainer feature-attribution scores over \texttt{Dataset 1}. The left panel reports normalized top-$6$ overlap, and the right panel reports Spearman correlation across all $72$ variables. Values are averaged over three random initializations.}
	\label{fig:agreement}
\end{figure*}

The comparisons involving PCA exhibit weak descriptive agreement. Rollout and PCA share few top-ranked variables, while GradientExplainer and PCA share, on average, approximately two of their six highest-ranked variables. These observations are summaries of the present cohort and training runs, not evidence of generalizable agreement.

Rollout and GradientExplainer identify between two and five common variables among their six highest-ranked features across the individual runs, with average normalized overlaps ranging from $0.50$ to $0.78$ across depths. Thus, in this dataset, their agreement is most visible at the top of the rankings. This descriptive comparison is sensitive to the selected top-$6$ threshold.

When the comparison is extended to the complete rankings, the Spearman correlations remain close to zero for all method pairs and depths. The top-$6$ overlap and the full-ranking comparison therefore describe different aspects of the same results: rollout and GradientExplainer share some highly ranked variables in this cohort while assigning substantially different orders to many variables of intermediate or low rank.

\subsection{Experimental results on Dataset 2}

As an auxiliary illustration, we repeat the main propagation analyses on the synthetic \texttt{Dataset 2}. This analysis is not intended to test generalization beyond \texttt{Dataset 1}; it only shows how the same diagnostics behave in a second tabular example.

The predictive performance obtained on \texttt{Dataset 2} is reported in Table \ref{tab:metrics_ds2}. As in \texttt{Dataset 1}, the point estimates vary little across the considered depths; no monotone association between depth and predictive performance is apparent.

\renewcommand{\arraystretch}{1.3}
\begin{table*}[!h]
    \centering
    \begin{tabular}{|c|c|c|c|c|}
        \hline \rowcolor{lightgray} \textbf{Depth} & $\text{\textbf{R}}^2$ & \textbf{RMSE} & \textbf{MAE} & \textbf{Best validation loss}
        \\
        \hline $3$ & $0.9035 \pm 0.0005$ & $6.4480 \pm 0.0152$ & $5.1568 \pm 0.0264$ & $0.0504 \pm 0.0001$
        \\
        \hline $5$ & $0.9039 \pm 0.0001$ & $6.4334 \pm 0.0045$ & $5.1136 \pm 0.0081$ & $0.0501 \pm 0.0001$
        \\
        \hline $7$ & $0.9039 \pm 0.0001$ & $6.4352 \pm 0.0013$ & $5.1140 \pm 0.0043$ & $0.0501 \pm 0.0001$
        \\
        \hline $9$ & $0.9041 \pm 0.0005$ & $6.4295 \pm 0.0159$ & $5.1076 \pm 0.0061$ & $0.0511 \pm 0.0004$
        \\
        \hline $11$ & $0.9042 \pm 0.0004$ & $6.4247 \pm 0.0122$ & $5.1229 \pm 0.0339$ & $0.0501 \pm 0.0002$
        \\
        \hline $13$ & $0.9037 \pm 0.0008$ & $6.4408 \pm 0.0271$ & $5.1403 \pm 0.0423$ & $0.0500 \pm 0.0002$
        \\
        \hline $15$ & $0.9032 \pm 0.0007$ & $6.4572 \pm 0.0228$ & $5.1402 \pm 0.0297$ & $0.0503 \pm 0.0005$
        \\
        \hline $17$ & $0.9034 \pm 0.0007$ & $6.4526 \pm 0.0251$ & $5.1540 \pm 0.0246$ & $0.0505 \pm 0.0002$
        \\
        \hline $20$ & $0.9034 \pm 0.0001$ & $6.4502 \pm 0.0042$ & $5.1386 \pm 0.0117$ & $0.0502 \pm 0.0001$
        \\
        \hline
    \end{tabular}
    \vspace{6pt}
    \caption{Training and predictive performance over \texttt{Dataset 2} of Transformer architectures with increasing depth. The table reports predictive accuracy on the test split for Transformer models with different numbers of self-attention layers. Results are presented as the mean $\pm$ standard deviation over independent runs using different random seeds, summarizing the variability observed across initializations. The reported metrics include the coefficient of determination $\text{R}^2$, Root Mean Squared Error (RMSE), Mean Absolute Error (MAE), and the best validation loss achieved during training.} \label{tab:metrics_ds2}
\end{table*}

While the predictive point estimates vary little, Figures \ref{fig:attn_rollout_ds2} and \ref{fig:dobrushin_ds2} show that the measured rollout matrices become more contractive with depth. This is a similar qualitative depth-dependent pattern to that observed for \texttt{Dataset 1}.
\begin{figure*}[!h]
	\centering
	\includegraphics[width=0.9\textwidth]{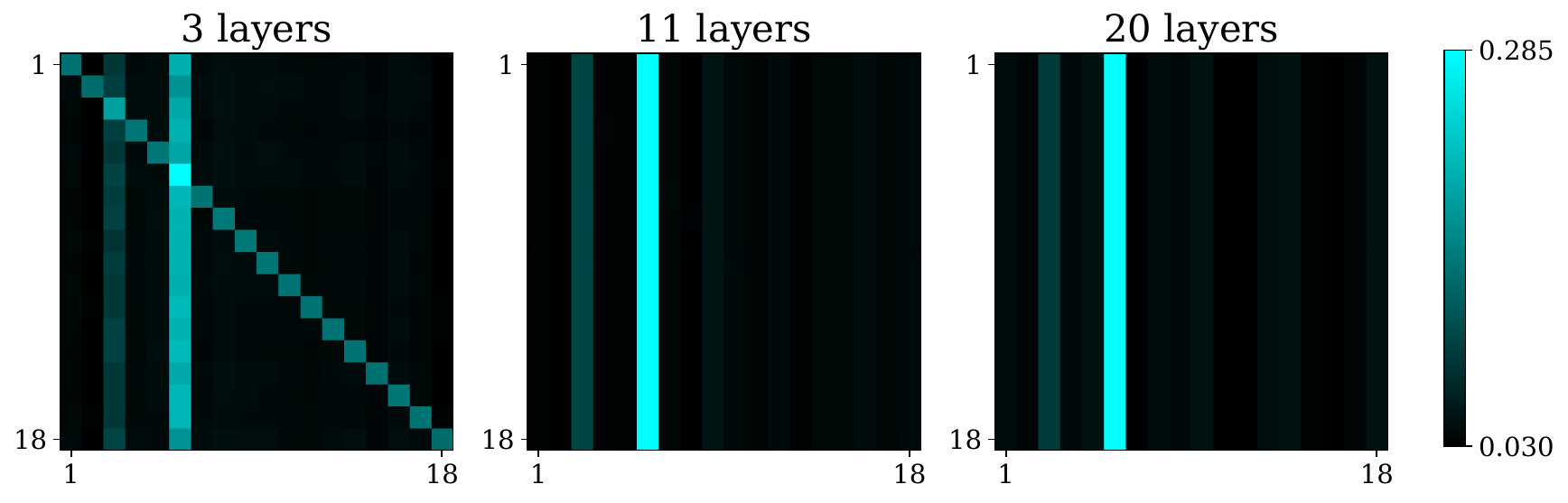}
	\caption{Global rollout matrices \eqref{eq:rollout_global} computed over the test split of \texttt{Dataset 2} for shallow (3-layer), intermediate (11-layer), and deep (20-layer) Transformers.}\label{fig:attn_rollout_ds2}
\end{figure*}

\begin{SCfigure}[0.6][ht]
	\centering
	\includegraphics[width=0.5\columnwidth]{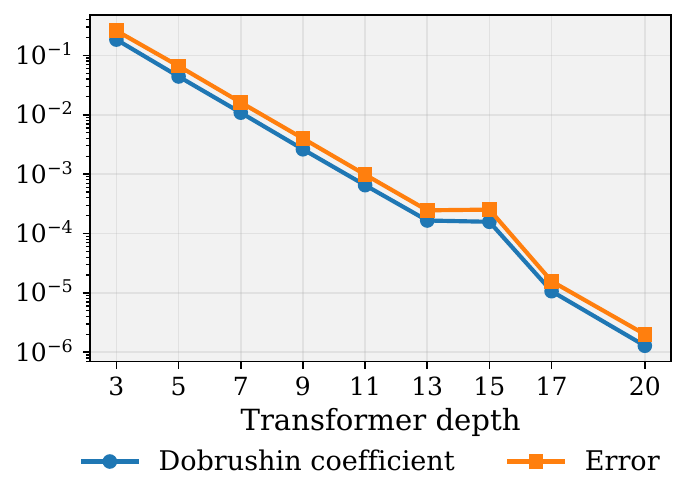}
	\caption{Evolution of the Dobrushin contraction coefficient $\kappa(\*R)$ and the approximation error $\|\*R-\Pi_{\*R}\|_{\infty,1}$ of the global rollout operator as functions of the Transformer depth on \texttt{Dataset 2}. Both quantities decrease rapidly with depth, showing that the rollout operator becomes progressively closer to its rank-one approximation, consistently with Theorem \ref{thm:rollout} and the depth-dependent estimate in Appendix \ref{app:proof}.}\label{fig:dobrushin_ds2}
\end{SCfigure}

After documenting the contraction and rank-one approximation pattern in this auxiliary example, we compare the three analytical viewpoints descriptively. Figure \ref{fig:agreement_ds2} reports the Spearman correlations between the rollout, GradientExplainer, and PCA rankings for Transformer architectures of increasing depth.

\begin{SCfigure}[0.5][ht]
	\centering
	\includegraphics[width=0.65\columnwidth]{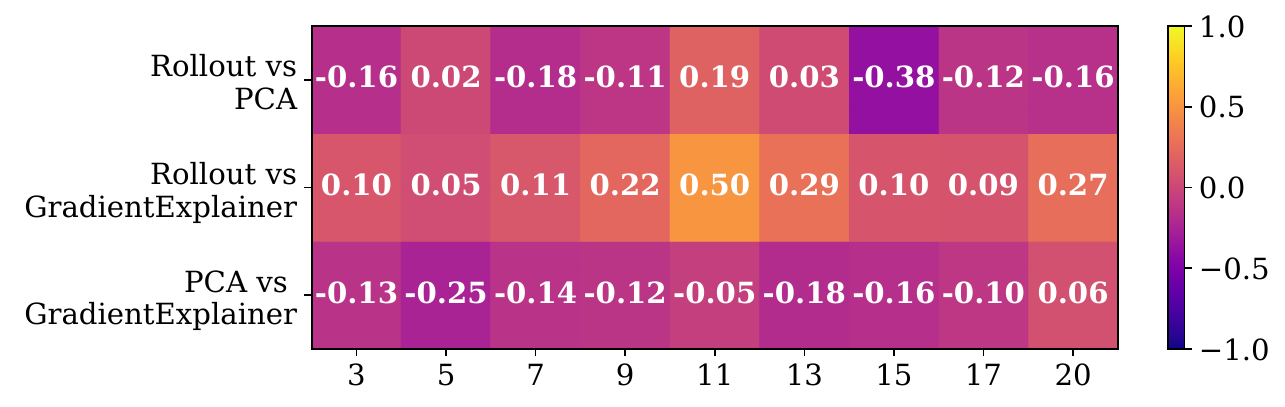}
	\caption{Comparison of rankings obtained with PCA scores, rollout propagation scores, and GradientExplainer feature-attribution scores over \texttt{Dataset 2} for Transformer architectures of varying depth. The heatmap reports the Spearman correlation computed over the complete rankings of all $18$ features.}\label{fig:agreement_ds2}
\end{SCfigure}

In this synthetic example, correlations involving PCA are weak, while rollout and GradientExplainer rankings have moderate positive agreement at several Transformer depths and remain far from identical. These observations describe this generated dataset only.

Overall, \texttt{Dataset 2} shows two patterns similar to the point estimates in \texttt{Dataset 1}: progressive contraction of the trained rollout operator and a non-uniform propagation profile. Because the dataset was generated by a language model rather than a fully specified stochastic generator with known ground truth, this experiment is an illustration rather than an independent validation or evidence of generality.

\subsection{Limitations of the empirical analysis}\label{subsec:limitations}

The empirical results should be interpreted within several limitations. First, attention rollout represents only propagation mediated by attention. It omits value and output projections, normalization layers, MLP sublayers, and other nonlinear transformations in the complete Transformer. Consequently, rollout propagation scores are neither causal effects nor guaranteed faithful attributions of the prediction. Second, the product analysis provides a layerwise upper bound on rollout contraction but does not by itself prove decay with depth; moreover, the experiments report trained contraction only, while the trained--random-initialization comparison concerns propagation profiles rather than contraction coefficients. Intervention-based controls would be required to establish the predictive relevance of individual rollout-ranked variables. Third, the real-data analysis uses a single biomedical cohort and three random initializations, so the ranking comparisons are descriptive and exploratory. Fourth, \texttt{Dataset 2} was generated with a language model rather than a fully specified stochastic data-generating process and does not provide ground-truth feature effects. Finally, the GradientExplainer background used in the reported experiments was sampled from the test set. It was employed exclusively for post-hoc interpretability after model training and was not involved in model fitting, hyperparameter selection, or checkpoint selection. Accordingly, the conclusions of this work concern the mathematical characterization and empirical illustration of the propagation encoded by attention rollout.

\section{Conclusions and perspectives}\label{sec:conclusions}

We introduced \emph{interior interpretability} as the analysis of explicitly defined operators that summarize selected interactions inside a learned architecture, and instantiated this perspective with attention rollout in tabular Transformers. The scope of the resulting explanation is operator-specific: rollout describes attention-mediated propagation, not the complete hidden-state dynamics and not a causal decomposition of the prediction.

Using a direct finite-dimensional consequence of classical Doeblin--Dobrushin theory, we characterized the rank-one regime of any row-stochastic rollout operator. Its Dobrushin coefficient controls the distance to a canonical rank-one matrix, and the normalized column masses are exactly the mean row profile. This identifies two mathematically distinct aspects of rollout: contraction determines how much destination-dependent variation is lost, whereas the profile records how the remaining propagated mass is distributed among source features. Neither rank-one structure nor a nonuniform profile, by itself, establishes learning, predictive relevance, or attribution faithfulness.

In the trained tabular Transformers studied here, the measured Dobrushin coefficient and rank-one approximation error decreased with depth. The trained and randomly initialized models also displayed different propagation profiles in the reported runs, but this experiment does not isolate the cause or predictive consequence of those differences. On \texttt{Dataset 1}, exploratory comparisons with PCA and GradientExplainer scores revealed localized overlap among some highly ranked variables and weak agreement across complete rankings. These observations motivate using rollout as a complementary diagnostic of propagation geometry while underscoring that variance, propagation, and prediction attribution answer different questions.

The broader methodological lesson is that a defensible interpretation of an internal quantity requires its underlying operator, assumptions, invariances, scope, and failure modes to be made explicit. This principle suggests a research program with four priorities:
\begin{enumerate}
    \item construct computation-aware propagation operators that incorporate value and output projections, residual paths, normalization, and MLP transformations, for example through local Jacobian or transport formulations;
    \item formulate propagation--attribution relationships as explicit, falsifiable hypotheses, supported by interventions and counterexamples rather than ranking agreement alone;
    \item quantify the formation, stability, and uncertainty of propagation profiles across initializations, subjects, cohorts, architectures, and appropriate random-model baselines;
    \item extend the contraction--profile decomposition to local, subgroup-specific, continuous-depth, and mean-field operators, where population heterogeneity and depth scaling can be studied directly.
\end{enumerate}

Interior interpretability is therefore proposed not as a universal feature-importance score, but as an operator-first framework for deriving mathematically scoped claims and empirical tests that can be falsified, refined, and compared across models.

\section*{Acknowledgments} The authors would like to express their sincere gratitude to Alain Ib\'a\~{n}ez de Opakua, Jos\'e Mar\'ia Mato and \'Oscar Millet (ATLAS Molecular Pharma, CIC bioGUNE, and CIBERehd) for providing the biomedical \texttt{Dataset 1} used in the experimental part of this work. The experimental investigations presented in this paper would not have been possible without their contribution in collecting, curating, and organizing the dataset, as well as their guidance concerning the biomedical context of the selected variables. The authors also wish to thank Ziqian Li (Chair for Dynamics, Control, Machine Learning, and Numerics at the Friedrich-Alexander-Universit\"at Erlangen-N\"urnberg) for his valuable assistance during the final revisions of this work.


\section*{Data availability} The data contained in \texttt{Dataset 1} were collected through individual medical examinations and are not publicly available due to privacy and ethical restrictions. Access to anonymized data may be granted upon reasonable request and subject to the approval of the relevant data-governance and ethical procedures. \texttt{Dataset 2} instead is available on the GitHub repository associated with this paper \cite{Github}.

\appendix
\section{Mathematical results}\label{app:proof}

This appendix collects the mathematical results supporting our discussion. 

\subsection{Structural properties of attention rollout} 

We first establish structural statements that apply to any row-stochastic propagation operator. They therefore apply both to each local rollout matrix $R(x)$ defined in \eqref{eq:rollout} and to the dataset-average global rollout $\*R$ defined in \eqref{eq:rollout_global}. Only the later depth-dependent estimate uses the product representation of the local rollout. Throughout the first part of the appendix, we consider a generic matrix
\begin{align}\label{eq:rollout_gen}
    \MC R\in\RR^{d\times d},
    \qquad&
    \MC R_{ij}\geq 0
    \quad\text{for every }i,j\in\inter{d},
    \notag\\
    &
    \sum_{j=1}^d\MC R_{ij}=1
    \quad\text{for every }i\in\inter{d}.
\end{align}

We start by giving the following definition of oscillation in a vector.

\begin{definition}\label{def:oscillations}
For every $x=(x_1,\ldots,x_d)\in\RR^d$, we define its oscillation by
\begin{align*}
    \operatorname{osc}(x) \coloneqq \max_{j\in\inter{d}}x_j - \min_{j\in\inter{d}}x_j.    
\end{align*}
\end{definition}

\noindent Notice that $\operatorname{osc}(x)$ defines a seminorm on $\RR^d$. 

The following result is the finite-dimensional specialization applied to the rollout matrix $\MC R$ of the Doeblin--Dobrushin contraction theorem (see \cite{dobrushin1956central} or \cite[Theorem 4.14]{seneta2006non}). We include it for completeness, since the proof is short and because the notation adopted here is tailored to attention rollout operators.

\begin{proposition}\label{prop:oscillations}
Let $\MC R=(\MC R_{ij})_{i,j=1}^d\in\RR^{d\times d}$ be a rollout matrix as in \eqref{eq:rollout_gen}, and let $\kappa(\MC R)$ denote its Dobrushin coefficient given by Definition \ref{def:dobrushin}. Then $0\leq \kappa(\MC R)\leq 1$ and, for every $x\in\RR^d$,
\begin{align*}
    \operatorname{osc}(\MC Rx) \leq \kappa(\MC R)\operatorname{osc}(x).    
\end{align*}
In particular, if there exists $\delta>0$ such that $\MC R_{ij}\geq \delta$ for all $i,j\in\inter{d}$, then $\delta\leq 1/d$,
\begin{align*}
    \kappa(\MC R)\leq 1-d\delta<1,
\end{align*}
and
\begin{align*}
    \operatorname{osc}(\MC Rx)
    \leq \kappa(\MC R)\operatorname{osc}(x)
    \leq \bigl(1-d\delta\bigr)\operatorname{osc}(x).
\end{align*}
\end{proposition}

\begin{proof}
Fix $x\in\RR^d$, and set
\begin{align*}
    m\coloneqq \min_{j\in\inter{d}}x_j, \quad M\coloneqq \max_{j\in\inter{d}}x_j.    
\end{align*}
Then $\operatorname{osc}(x)=M-m$. 

If $M=m$, then $\operatorname{osc}(x)=0$ meaning that $x$ is constant. In this case, since $\MC R$ is row-stochastic, we necessarily have $\MC Rx=x$, and therefore $\operatorname{osc}(\MC Rx)=0$.

Due to the above argument, in what follows we may assume that $M>m$. 

\medskip
\noindent Define the vector $y\in\RR^d$ by
\begin{align*}
    y_j\coloneqq \frac{x_j-m}{M-m}, \quad\text{ for all } j\in\inter{d}.    
\end{align*}
Since $m\leq x_j\leq M$, we have 
\begin{align*}
    0\leq y_j\leq 1, \quad\text{ for all } j\in\inter{d}.    
\end{align*}
Moreover, $x=m\mathbf 1+(M-m)y$. Because $\MC R$ is row-stochastic, $\MC R\mathbf 1=\mathbf 1$. Consequently,
\begin{align*}
    \MC R x = m\MC R\mathbf 1+(M-m)\MC R y = m\mathbf 1+(M-m)\MC R y.    
\end{align*}

Adding a constant vector does not change oscillation, while multiplication by a nonnegative scalar scales the oscillation. Therefore,
\begin{align}\label{eq:oscillation_prel}
    \operatorname{osc}(\MC R x) = (M-m)\operatorname{osc}(\MC R y).
\end{align}
It remains to estimate $\operatorname{osc}(\MC R y)$. For arbitrary $i,k\in\inter{d}$, we have
\begin{align}\label{eq:est_P}
    (\MC R y)_i-(\MC R y)_k = \sum_{j=1}^d \big(\MC R_{ij}-\MC R_{kj}\big)y_j.
\end{align}
For each $j\in\inter{d}$, define $c_j\coloneqq \min\{\MC R_{ij},\MC R_{kj}\}$. Then we may write
\begin{align*}
    \MC R_{ij} = c_j+p_j \quad\text{ and }\quad \MC R_{kj}=c_j+q_j,    
\end{align*}
where
\begin{align*}
    p_j\coloneqq \MC R_{ij}-c_j\geq 0 \quad\text{ and }\quad q_j\coloneqq \MC R_{kj}-c_j\geq 0.    
\end{align*}
Notice that, for every $j$, at least one of $p_j$ and $q_j$ is zero. Using this decomposition in \eqref{eq:est_P}, we obtain
\begin{align*}
    (\MC R y)_i-(\MC R y)_k = \sum_{j=1}^d p_jy_j - \sum_{j=1}^d q_jy_j.
\end{align*}
Moreover, since $0\leq y_j\leq 1$, it follows that
\begin{align*}
    0\leq \sum_{j=1}^d p_jy_j \leq \sum_{j=1}^d p_j \quad\text{ and }\quad 0\leq \sum_{j=1}^d q_jy_j \leq \sum_{j=1}^d q_j.   
\end{align*}
Therefore, using the row-stochasticity of $\MC R$,
\begin{align}\label{eq:est_P_2}
    (\MC R y)_i-(\MC R y)_k &\leq \sum_{j=1}^d p_j = \sum_{j=1}^d\bigl(\MC R_{ij}-c_j\bigr) =1-\sum_{j=1}^d c_j
    =1-\sum_{j=1}^d\min\{\MC R_{ij},\MC R_{kj}\}
    \notag\\
    &\leq
    1-\min_{r,q\in\inter{d}}
    \sum_{j=1}^d\min\{\MC R_{rj},\MC R_{qj}\} =\kappa(\MC R).
\end{align}
Since $i$ and $k$ are arbitrary, we can choose $i_\ast$ and $k_\ast$ such that
\begin{align*}
    (\MC R y)_{i_\ast} = \max_{i\in\inter{d}}(\MC R y)_i \quad\text{ and }\quad (\MC R y)_{k_\ast} = \min_{k\in\inter{d}}(\MC R y)_k.
\end{align*}
Applying \eqref{eq:est_P_2} to this pair yields
\begin{align}\label{eq:est_P_3}
    \operatorname{osc}(\MC R y) = (\MC R y)_{i_\ast}-(\MC R y)_{k_\ast} \leq \kappa(\MC R).
\end{align}
Combining \eqref{eq:oscillation_prel} and \eqref{eq:est_P_3}, we conclude that
\begin{align*}
    \operatorname{osc}(\MC R x) \leq (M-m)\kappa(\MC R) = \kappa(\MC R)\operatorname{osc}(x).    
\end{align*}
It remains to verify the properties of $\kappa(\MC R)$. For any pair $(i,k)$, we have
\begin{align*}
    0 \leq \sum_{j=1}^d\min\{\MC R_{ij},\MC R_{kj}\} \leq \sum_{j=1}^d \MC R_{ij} = 1.    
\end{align*}
This immediately gives $0\leq \kappa(\MC R)\leq 1$. If, moreover, $\MC R_{ij}>0$ for all $(i,j)$, then, for every $(i,k,j)$,
\begin{align*}
    \min\{\MC R_{ij},\MC R_{kj}\}>0.    
\end{align*}
Since there are finitely many rows and columns,
\begin{align*}
    \min_{i,k\in\inter{d}} \sum_{j=1}^d\min\{\MC R_{ij},\MC R_{kj}\} > 0.    
\end{align*}

It follows that $\kappa(\MC R)<1$. Finally, suppose that $\MC R_{ij}\geq\delta>0, \quad\text{for all }i,j\in\inter{d}$. Since $\MC R$ is row stochastic, we have $\delta\leq 1/d$. Moreover, for every pair $(i,k)$,
\begin{align*}
    \min\{\MC R_{ij},\MC R_{kj}\}\geq\delta,    
\end{align*}
and therefore
\begin{align*}
    \sum_{j=1}^d \min\{\MC R_{ij},\MC R_{kj}\} \geq d\delta.    
\end{align*}
Thus, $\kappa(\MC R) \leq 1-d\delta$ and we can conclude that $\operatorname{osc}(\MC R x) \leq \big(1-d\delta\big)\operatorname{osc}(x)$.    
\end{proof}

\noindent We now present the proof of Theorem \ref{thm:rollout}. 

\begin{proof}[Proof of Theorem \ref{thm:rollout}]
For all $i\in\inter{d}$, let
\begin{align*}
    r_i\coloneqq \big(\MC R_{i1},\ldots,\MC R_{id}\big) \in\RR^d 
\end{align*}
denote the $i$-th row of $\MC R$. Since $\MC R$ is row-stochastic, every $r_i$ belongs to the probability simplex
\begin{align*}
    \Delta^{d-1} \coloneqq \bigg\{z\in\RR^d:\; z_j\geq 0\ \text{for all }j\in\inter{d}, \sum_{j=1}^d z_j=1\bigg\}.
\end{align*}
Define 
\begin{align*}
    v\coloneqq \frac 1d \MC R^\top\*1 = \frac1d\sum_{i=1}^d r_i\in\RR^d.    
\end{align*}

Since $v$ is a convex combination of the rows $r_i\in\Delta^{d-1}$, and since the probability simplex is convex, we have that $v\in\Delta^{d-1}$. 

Now define $\Pi_{\MC R}\coloneqq\* 1v^\top$. Every row of $\Pi_{\MC R}$ is equal to $v^\top$. Therefore, 
\begin{align*}
    (\Pi_{\MC R})_{ij}=v_j\geq 0 \quad\text{ and }\quad \sum_{j=1}^d (\Pi_{\MC R})_{ij} = \sum_{j=1}^d v_j = 1
\end{align*}

Thus $\Pi_{\MC R}$ is row-stochastic. Moreover, since it is the outer product of the nonzero vectors $\*1$ and $v$, we also have that $   \operatorname{rank}(\Pi_{\MC R})=1$.    

It remains to estimate the distance between $\MC R$ and $\Pi_{\MC R}$. By \eqref{eq:dobrushin_def_eq}, the Dobrushin coefficient admits the representation
\begin{align*}
    \kappa(\MC R) = \frac 12 \max_{i,k\in\inter{d}} \|r_i-r_k\|_1.    
\end{align*}

Equivalently, for every $i,k\in\inter{d}$, we have $\|r_i-r_k\|_1 \leq 2\kappa(\MC R)$. Fix $i\in\inter{d}$. Since $v$ is the arithmetic mean of the rows, 
\begin{align*}
    r_i-v &= r_i-\frac 1d \sum_{k=1}^d r_k = \frac1d \sum_{k=1}^d(r_i-r_k).    
\end{align*}
Using the triangle inequality, then we obtain
\begin{align*}
    \|r_i-v\|_1
    \leq \frac1d\sum_{k=1}^d\|r_i-r_k\|_1
    =\frac1d\sum_{\substack{k=1\\k\neq i}}^d\|r_i-r_k\|_1
    \leq
    \frac1d\sum_{\substack{k=1\\k\neq i}}^d
    2\kappa(\MC R)
    =
    2\left(\frac{d-1}{d}\right)\kappa(\MC R).
\end{align*}
Since the $i$-th row of $\Pi_{\MC R}$ is $v^\top$, this gives 
\begin{align*}
    \|\MC R_{i\cdot} - (\Pi_{\MC R})_{i\cdot} \|_1 = \|r_i-v\|_1 \leq 2\left(\frac{d-1}{d}\right)\kappa(\MC R).    
\end{align*}
Taking the maximum over $i\in\inter{d}$, we conclude that
\begin{align*}
    \|\MC R-\Pi_{\MC R}\|_{\infty,1} \leq 2\left(\frac{d-1}{d}\right)\kappa(\MC R).    
\end{align*}
Finally, for the lower bound, notice that we can write $v$ as
\begin{align*}
    v=\frac{1}{d}\sum_{k=1}^d r_k.
\end{align*}
Therefore, for every $i,k\in\inter{d}$,
\begin{align*} 
    \|r_i-r_k\|_1 \leq \|r_i-v\|_1+\|r_k-v\|_1 \leq 2\|\MC R-\Pi_{\MC R}\|_{\infty,1}.
\end{align*}
Using again \eqref{eq:dobrushin_def_eq}, we conclude that $\kappa(\MC R) \leq \|\MC R-\Pi_{\MC R}\|_{\infty,1}$.
\end{proof}

Theorem \ref{thm:rollout} gives the rollout column scores a structural interpretation. When the Dobrushin coefficient is small, the rows of $\MC R$ are close to their mean profile $v=d^{-1} \MC R^\top \*1$, whose components satisfy $v_j=s_j/d$. Thus, in the approximately rank-one regime, the column scores encode the common distribution toward which the rows of the rollout operator homogenize. This statement concerns the representation of the rollout operator and does not imply that the profile $v$ is concentrated or that dominant coordinates necessarily emerge. Any heterogeneity of $v$ must be assessed separately from the contraction estimate.

\begin{remark}
We stress that small Dobrushin coefficient does not imply concentration of the propagation profile. For example, both $\MC R_1=d^{-1}\*1\*1^\top$ and $\MC R_2=\* 1 e_1^\top$, where $e_1$ is the first vector in the canonical basis of $\RR^d$, satisfy $\kappa(\MC R_i)=0$ and $\MC R_i=\Pi_{\MC R_i}$. However, the associated profiles are respectively uniform and maximally concentrated. Theorem \ref{thm:rollout} therefore characterizes row homogenization, but not the heterogeneity of the common profile.    
\end{remark}

\subsubsection{Depth-dependent rank-one approximation through the product structure}\label{app:product}

Theorem \ref{thm:rollout} characterizes the geometric structure of rollout operators in terms of their Dobrushin coefficient. We now examine how this coefficient is controlled by the product structure defining attention rollout. 

Throughout this subsection, fix $x\in\RR^d$ and write
\begin{align*}
    \MC R=R(x)
    =\MC A^{(L)}(x)\MC A^{(L-1)}(x)\cdots\MC A^{(1)}(x),
\end{align*}
where
\begin{align*}
    \MC A^{(\ell)}(x)
    \coloneqq \*A^{(\ell)}(x)
    =\frac12\Bigl(I+A^{(\ell)}(x)\Bigr).
\end{align*}

Here $A^{(\ell)}(x)$ is the attention matrix at layer $\ell$. To simplify notation, the dependence on $x$ is suppressed in the remainder of this subsection. Let $a_i^{(\ell)}$ denote the $i$-th row of $A^{(\ell)}$. Then, the $i$-th row of $\MC A^{(\ell)}$ is
\begin{align*}
    \* a_i^{(\ell)} = \frac12 \big(e_i+a_i^{(\ell)}\big),    
\end{align*}
where $e_i$ is the $i$-th canonical basis vector. Therefore, for every $i\neq k$, 
\begin{align*}
    \frac12 \left\|\*a_i^{(\ell)} - \*a_k^{(\ell)}\right\|_1 &= \frac 14 \left\|(e_i-e_k) + \big(a_i^{(\ell)}-a_k^{(\ell)}\big)\right\|_1
    \\
    &\leq \frac 14 \|e_i-e_k\|_1 + \frac 14 \left\|a_i^{(\ell)}-a_k^{(\ell)}\right\|_1 =\frac 12 + \frac 14 \left\|a_i^{(\ell)}-a_k^{(\ell)}\right\|_1.
\end{align*}
Taking the maximum over all pairs and using \eqref{eq:dobrushin_def_eq} gives
\begin{align}\label{eq:dobrushin_A_est}
    \kappa \left(\MC A^{(\ell)}\right) \leq \frac 12\Big(1+\kappa\left(A^{(\ell)}\right)\Big).    
\end{align}

Moreover, in the standard self-attention mechanism considered in this work, the attention matrices $A^{(\ell)}$ are obtained by applying a row-wise softmax to finite attention logits. Consequently, they have strictly positive entries. This implies $\kappa \left(A^{(\ell)}\right)<1$ and, consequently, $\kappa\left(\MC A^{(\ell)}\right)<1$.

Thus, every rollout factor is strictly contractive. The residual identity component nevertheless limits the strength of this contraction, as shown by the following lemma.

\begin{lemma}\label{lem:dobrushin_est}
Let $d\geq 2$, let $A\in\RR^{d\times d}$ be row-stochastic, and define
\begin{align*}
    P=\frac 12 (I+A).    
\end{align*}
Then
\begin{align*}
    \kappa(P) \geq \frac{d-2}{2(d-1)}.    
\end{align*}
\end{lemma}

\begin{proof}
Using Definition \ref{def:dobrushin}, we estimate the overlap between two distinct rows of $P$. For distinct $i,k\in\inter{d}$, we have
\begin{align*}
    P_{ii} = \frac 12 (1+A_{ii}) \quad\text{ and }\quad P_{ki} = \frac 12 A_{ki}.    
\end{align*}
Since $A_{ki}\leq 1\leq 1+A_{ii}$, it follows that
\begin{align*}
    \min\{P_{ii},P_{ki}\} = \frac 12 A_{ki} \quad\text{ and }\quad \min\{P_{ik},P_{kk}\} = \frac 12 A_{ik}.
\end{align*}
For $j\notin\{i,k\}$, one therefore has
\begin{align*}
    \min\{P_{ij},P_{kj}\} = \frac 12 \min\{A_{ij},A_{kj}\} \leq \frac 14 (A_{ij}+A_{kj}).    
\end{align*}
Consequently,
\begin{align*}
    \sum_{j=1}^d \min\{P_{ij},P_{kj}\} \leq \frac 12 (A_{ki}+A_{ik}) + \frac 14\sum_{j\notin\{i,k\}} (A_{ij}+A_{kj}) = \frac 12 + \frac 14 \big(A_{ki}+A_{ik}-A_{ii}-A_{kk}\big).    
\end{align*}
Averaging the last quantity over unordered pairs $i<k$, we obtain
\begin{align*}
    \binom d2^{-1} \sum_{i<k} \big(A_{ki}+A_{ik}-A_{ii}-A_{kk}\big) = \frac{2\big(1-\operatorname{tr}(A)\big)}{d-1} \leq \frac{2}{d-1},    
\end{align*}
where we used the row-stochasticity of $A$ and $\operatorname{tr}(A)\geq 0$. Hence there exists at least one pair $i\neq k$ such that
\begin{align*}
    \sum_{j=1}^d\min\{P_{ij},P_{kj}\} \leq \frac 12 +\frac{1}{2(d-1)} = \frac{d}{2(d-1)}.    
\end{align*}
It follows that
\begin{align*}
    \kappa(P) \geq 1-\frac{d}{2(d-1)} = \frac{d-2}{2(d-1)}.    
\end{align*}
\end{proof}

Lemma \ref{lem:dobrushin_est} provides a dimension-dependent hard floor for the Dobrushin coefficient of an individual rollout factor. In particular,
\begin{align*}
    \frac{d-2}{2(d-1)} = \frac12-\frac{1}{2(d-1)}\to \frac 12 \text{ as } d\to+\infty.    
\end{align*}

Thus, because one half of each rollout factor is reserved for the identity, a single layer cannot contract row oscillations by substantially more than a factor of one half. 

Finally, applying submultiplicativity to the local product above and using \eqref{eq:dobrushin_A_est} (see, for instance, \cite[Section 4.3]{seneta2006non}), we obtain
\begin{align}\label{eq:dobrushin_R_est}
    \kappa(\MC R) \leq \prod_{\ell=1}^L \kappa \left(\MC A^{(\ell)}\right) \leq \prod_{\ell=1}^L \frac{1+\kappa(A^{(\ell)})}{2} = 2^{-L} \prod_{\ell=1}^L \left(1+\kappa(A^{(\ell)})\right).    
\end{align}

Estimate \eqref{eq:dobrushin_R_est} separates the respective roles of residual averaging and of the layerwise attention patterns. The factor $2^{-L}$ originates from the residual averaging in each rollout factor, while the terms $1+\kappa(A^{(\ell)})$ quantify how the attention matrices modulate the resulting bound. Since strict positivity only guarantees $\kappa(A^{(\ell)})<1$ separately at each layer, \eqref{eq:dobrushin_R_est} does not by itself imply that $\kappa(\MC R)\to 0$ as $L\to +\infty$. Such a conclusion would follow, for instance, under the uniform condition $\sup_{\ell\geq 1}\kappa(A^{(\ell)})\leq\delta<1$, in which case
\begin{align*}
    \kappa(\MC R)\leq\left(\frac{1+\delta}{2}\right)^L\to 0.  
\end{align*}

In the absence of such a condition, the depth-dependent decrease of the Dobrushin coefficient remains an empirical property of the architectures studied in Section \ref{sec:experiments}. Whenever the small-coefficient regime is reached, Theorem \ref{thm:rollout} guarantees that the rollout operator is close to a rank-one stochastic matrix whose propagation profile is determined by the realized attention matrices.

\begin{remark}
The product estimate \eqref{eq:dobrushin_R_est} applies to the local rollout operator \eqref{eq:rollout} which is by definition a product of attention propagation matrices. The global rollout operator $\*R$ in \eqref{eq:rollout_global} is instead defined as the dataset average and is generally not representable as a product of layerwise rollout matrices. Nevertheless, $\*R$ is row-stochastic because it is the average of row-stochastic matrices. Consequently, Theorem \ref{thm:rollout} applies directly to $\*R$, independently of any product representation.

Furthermore, by \cite[Theorem 5.1]{gaubert2015dobrushin}, the Dobrushin coefficient coincides with the operator norm induced by Hopf's oscillation seminorm. Hence, by the convexity of operator norms,
\begin{align*}
    \kappa(\*R) = \kappa\left(\frac1N\sum_{n=1}^N R(x_n)\right) \leq \frac1N\sum_{n=1}^N\kappa(R(x_n)).    
\end{align*}
Combining this inequality with \eqref{eq:dobrushin_R_est} yields
\begin{align}\label{eq:dobrushin_R_est_global}
    \kappa(\*R) \leq \frac{2^{-L}}N\sum_{n=1}^N \prod_{\ell=1}^L \Big(1+\kappa(A^{(\ell)}(x_n))\Big).    
\end{align}
\end{remark}

\subsection{A simplified analytical relation between rollout and SHAP}

To make the analytical statement unambiguous, we work with exact interventional Shapley values. Let $D\coloneqq\inter{d}$, let $Z\sim\mu$ be a reference random vector, and, for $S\subseteq D$ and $x\in\RR^d$, define the hybrid vector
\begin{align*}
    \bigl(x_S\oplus Z_{D\setminus S}\bigr)_k
    \coloneqq
    \begin{cases}
        x_k, & k\in S,\\
        Z_k, & k\notin S.
    \end{cases}
\end{align*}
The interventional coalition game is
\begin{align*}
    V_f^{\mathrm{int}}(S;x)
    \coloneqq
    \mathbb E\!\left[
        f\bigl(x_S\oplus Z_{D\setminus S}\bigr)
    \right].
\end{align*}
The corresponding interventional Shapley value \cite{shapley1953value} is
\begin{align}\label{eq:interventional_shap}
    \phi_{f,j}^{\mathrm{int}}(x) \coloneqq 
    \sum_{S\subseteq D\setminus\{j\}}
    \frac{|S|!\,(d-|S|-1)!}{d!}
    \Bigl(
        V_f^{\mathrm{int}}(S\cup\{j\};x)
        -V_f^{\mathrm{int}}(S;x)
    \Bigr).
\end{align}
Whenever the relevant expectations exist, efficiency gives
\begin{align*}
    f(x)
    =
    \mathbb E[f(Z)]
    +\sum_{j=1}^d\phi_{f,j}^{\mathrm{int}}(x).
\end{align*}
This game differs in general from conditional SHAP, which is based on
\begin{align*}
    V_f^{\mathrm{cond}}(S;x)
    \coloneqq
    \mathbb E\!\left[f(Z)\mid Z_S=x_S\right],
\end{align*}
whose associated conditional Shapley values $\phi_{f,j}^{\mathrm{cond}}(x)$ are obtained by applying to it the Shapley-value construction \eqref{eq:interventional_shap}.

If the components of $Z$ are independent, the interventional and conditional games coincide, up to the usual almost-everywhere qualification. The independence assumption in Proposition \ref{prop:shap} is imposed for this reason and should be regarded as a restrictive sufficient condition.

To study how propagation- and attribution-based quantities can become related in a simplified setting, we fix a rollout matrix $\MC R\in\RR^{d\times d}$ that is independent of the input $x$, and consider
\begin{align}\label{eq:f_split}
    f(x)=g(x)+e(x), \quad g(x)=\beta^\top\MC R x, \quad \beta\in\RR^d.
\end{align}

This decomposition is an analytical device, not a surrogate representation of the complete Transformer. Expanding $g$ gives
\begin{align*}
    g(x)=\sum_{j=1}^d\gamma_jx_j,
    \qquad
    \gamma_j=\sum_{i=1}^d\beta_i\MC R_{ij}.
\end{align*}

The following proposition supplies only a one-sided bound. It does not imply agreement between rollout and SHAP rankings.

\begin{proposition}\label{prop:shap}
Let $\MC R$ be a fixed, input-independent, rollout operator, let $\beta\in\RR^d$, and let $f=g+e$ be as in \eqref{eq:f_split}. Define
\begin{align*}
    \beta^+
    \coloneqq
    \max_{i\in\inter{d}}|\beta_i|,
    \qquad
    s_j
    \coloneqq
    \sum_{i=1}^d\MC R_{ij}.
\end{align*}

Let $X\sim\mu$ have independent components satisfying $\mathbb E[X_j]=0$ and $\mathbb E[|X_j|]<+\infty$ for all $j\in\inter{d}$,
and assume that $e$ is globally bounded on $\RR^d$. Write
\begin{align*}
    \|e\|_\infty
    \coloneqq
    \sup_{x\in\RR^d}|e(x)|
    <+\infty.
\end{align*}
Define the exact interventional SHAP importance by
\begin{align*}
    I_{f,j}^{\mathrm{int}}
    \coloneqq
    \mathbb E\!\left[
        \left|\phi_{f,j}^{\mathrm{int}}(X)\right|
    \right].
\end{align*}
Then
\begin{align}\label{eq:SHAP_est_above}
    I_{f,j}^{\mathrm{int}}
    \leq
    \beta^+s_j\,\mathbb E[|X_j|]
    +2\|e\|_\infty.
\end{align}
Under the stated independence assumption, the same bound holds for exact conditional SHAP values.
\end{proposition}

\begin{proof}
The interventional Shapley operator is linear, so
\begin{align*}
    \phi_{f,j}^{\mathrm{int}}(x) =\phi_{g,j}^{\mathrm{int}}(x) + \phi_{e,j}^{\mathrm{int}}(x).
\end{align*}

For every coalition $S$ and every $x\in\RR^d$, $\left|V_e^{\mathrm{int}}(S;x)\right| \leq \|e\|_\infty$. The Shapley weights in \eqref{eq:interventional_shap} are nonnegative and sum to one. Hence $\left|\phi_{e,j}^{\mathrm{int}}(x)\right| \leq 2\|e\|_\infty$
for every $x\in\RR^d$. For the linear function $g$,
\begin{align*}
    V_g^{\mathrm{int}}(S;x)
    =
    \sum_{k\in S}\gamma_kx_k
    +
    \sum_{k\notin S}\gamma_k\mathbb E[X_k].
\end{align*}
Every marginal contribution of feature $j$ is therefore
\begin{align*}
    \gamma_j\bigl(x_j-\mathbb E[X_j]\bigr)
    =\gamma_jx_j,
\end{align*}
and consequently $\phi_{g,j}^{\mathrm{int}}(x)=\gamma_jx_j$. Since rollout matrices are nonnegative,
\begin{align*}
    |\gamma_j|
    &=
    \left|\sum_{i=1}^d\beta_i\MC R_{ij}\right|
    \leq
    \sum_{i=1}^d|\beta_i|\MC R_{ij}
    \leq
    \beta^+s_j.
\end{align*}
Thus, for every $x\in\RR^d$,
\begin{align*}
    \left|\phi_{f,j}^{\mathrm{int}}(x)\right|
    \leq
    \beta^+s_j|x_j|
    +2\|e\|_\infty.
\end{align*}
Taking expectations proves \eqref{eq:SHAP_est_above}.
\end{proof}

\begin{remark}
Nonnegativity of $\MC R$ is essential for the bound in terms of the unsigned column mass $s_j$. Row-stochasticity is not otherwise used in this proposition. For an arbitrary signed matrix $B$, the same argument requires the absolute column mass
\begin{align*}
    t_j\coloneqq\sum_{i=1}^d|B_{ij}|,
\end{align*}
and yields a bound with $t_j$ in place of $s_j$. The proposition therefore does not extend verbatim to arbitrary signed matrices while retaining the original score $s_j$.
\end{remark}

The independence assumption is substantive and is generally not satisfied by correlated biomedical covariates. Centering can enforce $\mathbb E[X_j]=0$, but centering, scaling, whitening, PCA, and decorrelation do not in general guarantee independence. Proposition \ref{prop:shap} should therefore be read only as an illustrative sufficient-condition result; it does not theoretically explain or validate the empirical agreement between rollout rankings and \texttt{GradientExplainer} rankings.

Finally, because $d$ is finite, the first-moment assumptions imply that
\begin{align*}
    M\coloneqq\max_{j\in\inter{d}}\mathbb E[|X_j|]<+\infty,
\end{align*}
and hence
\begin{align*}
    I_{f,j}^{\mathrm{int}}
    \leq
    \beta^+Ms_j
    +2\|e\|_\infty.
\end{align*}

\section{Biomedical context for selected highly ranked variables}\label{app:bio}

This appendix places selected variables from the representative 11-layer ranking in biomedical context. It is a literature cross-check, not a validation of the ranking, a biomarker-discovery analysis, or evidence of causal physiological mechanisms.

The discussion is based on the representative realization presented in Figure \ref{fig:importance}; the rankings vary across training realizations (Section \ref{subsec:feature_importance}). Several selected variables can be related to established biomedical literature:
\begin{itemize}
    \item Serum albumin has been studied as a marker of nutritional and inflammatory status and is associated with health outcomes in older adults \cite{cabrerizo2015serum}.
    \item LDL cholesterol is an established causal risk factor for atherosclerotic cardiovascular disease \cite{ference2017low}.
    \item Glucose is central to glycemic status. The metabolite 1,5-AG is used as a marker of short-term glycemic excursions \cite{ceriello2019glycaemic,dungan20061}.
    \item 3-hydroxyisobutyric acid has been reported to promote vascular fatty-acid transport and to contribute to insulin resistance in experimental models \cite{jang2016branched}.
    \item Erythrocyte sedimentation rate is a nonspecific clinical marker of inflammation. Its appearance is compatible with, but does not establish, a link to the literature on chronic low-grade inflammation in aging \cite{ferrucci2018inflammageing}.
\end{itemize}

The appearance of these highly ranked variables is compatible with prior biomedical literature, but it does not rule out confounding, redundancy, or spurious association. The partial overlap between attribution-based and propagation-based rankings is therefore treated only as a hypothesis-generating observation.

The rollout operator should not be interpreted as reconstructing biochemical pathways or causal physiological mechanisms. The literature comparison supplies a qualitative plausibility check for this particular realization, not an independent validation of the framework or of any biomarker effect.



\bibliographystyle{siamplain}
\bibliography{biblio}


\end{document}